\documentclass[lettersize,journal]{IEEEtran}
\usepackage{amsmath,amsfonts}
\usepackage{algorithmic}
\usepackage{algorithm}
\usepackage{array}
\usepackage[caption=false,font=normalsize,labelfont=sf,textfont=sf]{subfig}
\usepackage{textcomp}
\usepackage{stfloats}
\usepackage{url}
\usepackage{verbatim}
\usepackage{graphicx}
\usepackage{cite}
\usepackage{mathtools}
\usepackage{multicol}
\usepackage{multirow}
\usepackage{pifont} 
\usepackage{dsfont} 
\usepackage{amsmath}
\usepackage{amssymb}
\usepackage{amsthm} 

\usepackage{booktabs}   
\usepackage{tabularx}  
\usepackage{amsmath}   
\usepackage[table]{xcolor} 
\definecolor{mygreen}{RGB}{223,239,218}
\definecolor{myred}{RGB}{187,39,26}
\definecolor{lightyellow}{RGB}{255,248,220}

\definecolor{mygray}{rgb}{.95,.95,.95}

\newcommand{\biogap}{\vspace{-0.8\baselineskip}}

\definecolor{myblue}{RGB}{225,235,245}
\definecolor{lightblue}{RGB}{242,245,249}

\definecolor{ForestGreen}{RGB}{34,139,34}  
\newcommand{\upgain}[1]{{\scriptsize\color{darkgray}(#1)\,\color{ForestGreen}$\uparrow$}}
\newcommand{\downdrop}[1]{{\scriptsize\color{darkgray}(#1)\,\color{red}$\downarrow$}}
\newcommand{\zerogain}[1]{{\scriptsize\color{darkgray}(#1)}}

\newtheorem{proposition}{Proposition}

\begin{document}

\title{Advantage-Guided Gate: Reshaping Open-Ended Reasoning for Vision-Based Spatial Intelligence}

\author{Ling Lin, Yang Bai, Congcong Zhu, Jiangming Shi, Meng Wang, Yang Long,~\IEEEmembership{Senior Member~IEEE}, \\ Jingrun Chen, Ling Shao,~\IEEEmembership{Fellow~IEEE}, Huazhu Fu,~\IEEEmembership{Senior Member~IEEE}

\IEEEcompsocitemizethanks{
		
\IEEEcompsocthanksitem

        Ling Lin, Congcong Zhu, and Jingrun Chen are with the School of Artificial Intelligence and Data Science, University of Science and Technology of China, Hefei 230000, China, and Suzhou Institute for Advanced Research, USTC, Suzhou 215123, China. E-mail: linglin00d@gmail.com; cczly@ustc.edu.cn; jingrunchen@ustc.edu.cn.

        Yang Bai and Huazhu Fu are with Institute of Advanced Intelligence and Computing (IAIC), Agency for Science, Technology and Research (A*STAR), Singapore (E-mail: bai\_yang@a-star.edu.sg; hzfu@ieee.org).

        Jiangming Shi is with the School of Computer Science and Technology, East China Normal University, Shanghai, China (E-mail: jmshi@cs.ecnu.edu.cn).

        Meng Wang is with the Centre for Innovation \& Precision Eye Health, Department of Ophthalmology, Yong Loo Lin School of Medicine, National University of Singapore, Singapore (E-mail: wang.m@nus.edu.sg).

        Yang Long is with the Department of Computer Science, Durham University, United Kingdom (E-mail: yang.long@durham.ac.uk).

        Ling Shao is with the UCAS-Terminus AI Lab, University of Chinese Academy of Sciences, Beijing, China  (E-mail:ling.shao@ieee.org).
		
		(Corresponding author: Yang Bai)
		
}
        
}



\maketitle

\begin{abstract}
Multimodal large language models (MLLMs) have demonstrated significant potential in complex spatial scene understanding and reasoning tasks. However, their open-ended reasoning process is prone to decision errors and error accumulation, leading to instability in answer quality. To address this, we propose an advantage-guided gating framework that dynamically intervenes in and corrects deviations during the reasoning process. Specifically, we model step-by-step reasoning as a finite-horizon decision process and introduce Monte Carlo value evaluation on the reasoning tree to provide intermediate supervision signals. The framework includes Step-Advantage Gate and Trajectory-Advantage Gate, which dynamically select high-value reasoning steps and high-quality complete reasoning trajectories, respectively. During training, we perform supervised learning for the gates using reasoning trees generated via multi-branch sampling, and combine shared-parameter initialization with task-specific heads to achieve cross-task robustness and diversity. During inference, the model greedily selects high-value prefix reasoning steps while choosing the optimal reasoning head based on the problem type, thereby significantly improving the accuracy of the final answer. Furthermore, we constructed the Reasoning-Tree-160k dataset and performed two-stage learning on it. Extensive experiments demonstrate that this advantage-guided gating framework effectively enhances the performance of benchmark MLLMs in visual-based spatial understanding and reasoning tasks. The code is open to the public for research: \url{https://github.com/LingLin-ll/Advantage-Guided-Gate}. 
\end{abstract}


\begin{figure}[t]
    \centering
    \includegraphics[width=\linewidth]{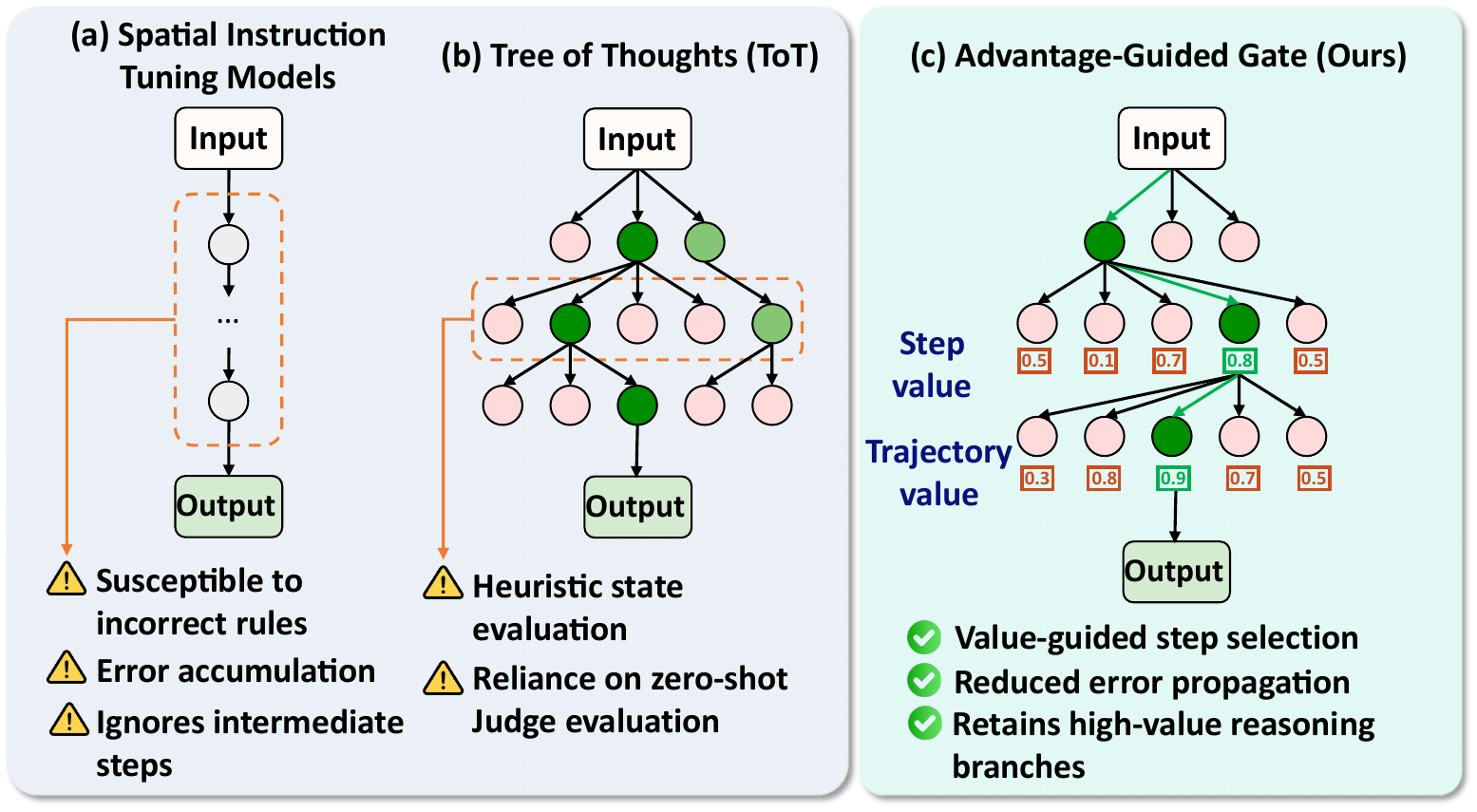}
    \caption{(a) Spatial instruction-tuned models generate a single autoregressive reasoning chain without explicitly evaluating intermediate-step quality, making them susceptible to incorrect spatial rules and accumulated errors. (b) Tree of Thoughts (ToT) explores multiple candidate branches, but its state selection relies on heuristic, zero-shot judge evaluation. (c) Our Advantage-Guided Gate estimates step- and trajectory-level values to suppress low-value intermediate steps and guide generation toward higher-value reasoning branches, thereby reducing downstream error propagation. Circles denote intermediate reasoning steps. Green circles and arrows indicate the selected reasoning steps and branches, while pink circles represent candidate nodes that are not selected for the final reasoning trajectory. The numerical scores denote the estimated prefix step or trajectory values.}
    \label{fig: cover figure}
    \vspace{-0.5cm}
\end{figure}

\section{Introduction}
\IEEEPARstart{M}{ultimodal} large models demonstrate significant potential in understanding and reasoning about complex environments. In spatial scenes, such models must not only perform precise spatial reasoning but also execute subsequent instructions based on the results of that reasoning; this capability is a critical link in effectively connecting general artificial intelligence (AGI) with real-world applications~\cite{zhang2026embodied3dbench,yang2025mmsi,cong2023reltr,gao2023scenehgn, wei2025planner3d,tan2023knowledge}. Current research primarily follows two approaches to data acquisition: The first assumes that the model has access to additional 3D or 2.5D information (such as point clouds~\cite{deng20253d, huang2024chat, chen2024ll3da}, camera parameters, or depth maps~\cite{lu2025matrix3d, liu2025worldmirror, zhu2025llava, zheng2025video}) in addition to traditional 2D visual inputs (images or videos). However, this reliance on additional data sources significantly limits the models’ broad applicability in real-world scenarios. The second category of research relies solely on monocular video inputs from the scene; under these conditions, the model’s spatial understanding and reasoning capabilities are typically referred to as “vision-based spatial intelligence”~\cite{li2025sti, yang2025thinking, zhao2025spacemind, zheng2026learning, wu2026spatial, wang2026s3tir}. Although some progress has been made in this domain, the reliability of reasoning by existing multimodal large models in 3D environments that require complex spatial reasoning remains unsatisfactory.

The core challenge lies in the deviation of reasoning mechanisms during video-based spatial reasoning, where local discrepancies in intermediate reasoning stages are propagated and amplified, leading to severe error accumulation along the reasoning trajectory. This deficiency is intrinsically linked to the autoregressive formulation of MLLMs. When conditioned on multimodal inputs, early erroneous outputs distort the soft probability distributions of subsequent tokens, thereby introducing permanent biases into the remaining Chain of Thought (CoT) trajectory. Under conventional end-to-end generative paradigms, once such intermediate biases are sampled, they are mathematically irreversible, leading to inevitable error amplification. A natural solution is to shift from holistic generation to a model that combines step-by-step CoT synthesis with step-level verification. However, implementing this fine-grained control faces a fundamental theoretical bottleneck: how to precisely evaluate the fidelity of generated steps within an unconstrained, nearly infinite search space? Compared to structured reasoning (such as mathematical derivations, where intermediate steps can be formally verified), open-ended spatiotemporal reasoning lacks deterministic verification trajectories, making it extremely challenging to estimate the quality of intermediate reasoning states.

To mitigate the aforementioned challenges, we propose an Advantage-Guided Gating framework. As illustrated in Fig.~\ref{fig: cover figure}, our method differs from existing spatial reasoning paradigms by explicitly suppressing low-value intermediate steps through learned value guidance. Specifically, we formulate step-wise reasoning as a finite-horizon Markov decision process (MDP), where each reasoning step constitutes an action, and the final correctness of the answer defines the terminal reward. Within this framework, an ideal step-level gate corresponds to a value-based policy improvement operator. Given that the true action-value function is latent and not directly observable, we construct reasoning trees and estimate prefix values via Monte Carlo rollouts. To substantiate this approach and facilitate training, we systematically constructed the Reasoning-Tree-160k dataset by generating extensive reasoning trees. Driven by the law of large numbers, the empirical subtree return asymptotically converges to the true action value as the rollout budget scales. To parameterize this mechanism, we introduce two distinct variants: the trajectory-advantage gate, which estimates trajectory-level terminal returns, and the step-advantage gate, which learns a discrete approximation of the ideal value-based gate. Theoretically, the performance gap between the ideal and the learned gating mechanisms can be decomposed into Monte Carlo estimation error and classification approximation error. Finally, we empirically and theoretically verify the performance gains using a lower confidence bound (LCB) on held-out task rewards.

Our main contributions can be summarized as the following: 

\begin{itemize}
    \item We establish a theoretical framework for advantage-guided open-ended reasoning and prove that advantage-guided gating yields policy improvement and reasoning-space contraction toward high-quality solution regions.
    \item We construct and release a large-scale reasoning tree dataset containing over 160K reasoning states and trajectories, providing hierarchical supervision signals for studying open-ended multimodal reasoning.
    \item We propose a hierarchical advantage supervision mechanism that derives intermediate reasoning preferences from terminal outcomes through multi-branch reasoning tree construction and backward propagation.
    \item We develop a plug-and-play reasoning control framework that reshapes inference-time reasoning distributions without modifying the underlying MLLM parameters, enabling reasoning enhancement.
\end{itemize}

\section{Related Work}
\subsection{Spatial Reasoning in Multimodal Large Language Models}
The emergence of MLLMs has significantly enhanced cross-modal understanding capabilities~\cite{zou2025uncertainty,an2025inter,zhang2026world2vlm,bai2024generalist,an2025striving,bai2024sentence,an2026biprolora}. In recent years, endowing MLLMs with spatial reasoning capabilities—specifically, the ability to perceive, localize, and reason logically about the three-dimensional physical world—has become a hot topic in academic research~\cite{liu2026openspatial,ma2026causalspatial,Gong2022Person,Gong2026Theory}. Based on the underlying geometric representations, existing research primarily follows two technical approaches: enhancement mechanisms based on explicit 3D representations, and spatial intelligence driven by monocular video.

\textbf{Enhancement mechanisms based on explicit 3D representations.} A mainstream research approach focuses on enhancing MLLMs' spatial awareness by introducing explicit 3D geometric primitives into traditional 2D inputs. For example, Hong et al.~\cite{hong20233d} proposed the 3D-LLM framework, which addresses the model's lack of global spatial awareness by reconstructing 3D point clouds from multi-view images and aligning them semantically with text. Xu et al.~\cite{Xu2024Point} developed PointLLM. This model incorporates a native 3D encoder that directly processes raw color point clouds, significantly improving accuracy in 3D bounding box prediction and object classification. Ji et al.~\cite{ji2024jm3d} proposed the JM3D framework, which achieves deep fusion of point cloud, text, and image features through a Structured Multimodal Organizer (SMO) and Joint Multi-modal Alignment (JMA).

\textbf{Spatial intelligence driven by monocular video.} Although explicit 3D methods have made significant progress, they are often constrained by the high cost of acquiring 3D data and rely heavily on specialized sensors. To overcome these physical limitations, another paradigm has emerged that focuses on enabling spatial reasoning directly from ubiquitous monocular video or 2D observation data. Wherever possible, Huang et al.~\cite{huang2026chat} proposed the Chat-Scene++ framework, which structures complex 3D scenes into context-aware object sequences, enabling object-centric spatial representation and interaction. Similarly, Wu et al.~\cite{wu2026spatial} introduced the novel Spatial-MLLM framework, designed to perform visual-spatial reasoning under the constraint of purely 2D observational data. To completely eliminate the reliance on explicit 3D priors, Chen et al.~\cite{chen2026think} proposed the 3DThinker framework, which effectively extracts and utilizes the rich geometric information implicitly embedded in images to enable 3D reasoning. Furthermore, to rigorously benchmark this paradigm, Yang et al.~\cite{yang2025thinking} constructed VSI-Bench, a visual-spatial intelligence benchmark dataset tailored specifically for video-based spatial reasoning. This monocular video paradigm relies solely on widely available single-camera footage, offering exceptional flexibility and scalability, and is considered a more viable path toward general-purpose open-world spatial intelligence.

\subsection{Path Search and Step Evaluation in LLM Reasoning}
To enhance the complex reasoning capabilities of large language models, Chain-of-Thought enables step-by-step task solving by generating intermediate reasoning steps~\cite{wei2022chain}. Although CoT significantly improves model interpretability and accuracy, its inherently single-path autoregressive generation makes it highly susceptible to inheriting and amplifying biases from earlier stages, rendering it impossible to avoid catastrophic error accumulation. To broaden the search space, the Tree of Thoughts (ToT)~\cite{yao2023tree} extends the reasoning process into a tree-based search, generating multiple candidate lines of thought and filtering them through scoring to achieve a more comprehensive exploration of reasoning. Furthermore, the Graph of Thoughts (GoT)~\cite{besta2024graph} generalizes the reasoning topology to a directed graph, enabling different thinking units to combine and aggregate across paths, thereby enhancing the flexibility of reasoning and the ability to reuse information. However, the scoring of intermediate steps during reasoning in these classical path-search methods primarily relies on the LLM’s own zero-shot prompts. This heuristic evaluation mechanism not only lacks clear training supervision objectives but also often lacks direct statistical correlation between the scores and the quality of the final answer~\cite{lightman2024let}.

To move beyond heuristic evaluation, recent studies have explored explicit supervision for intermediate reasoning quality. Lightman et al.~\cite{lightman2024let} introduced a Process Reward Model (PRM) based on large-scale manual annotation for stepwise verification of mathematical reasoning. However, such annotation is costly and subjective, and its reliance on explicit correctness criteria limits generalization to open-ended reasoning without gold intermediate steps. Math-Shepherd~\cite{wang2024math} reduces annotation costs by generating PRM labels through large-scale sampling, deterministic rules, and LLM-based error localization, but still assumes that intermediate states are directly verifiable. More recently, Wang et al.~\cite{wang2025towards} proposed a Hierarchical Reward Model that estimates state values through online Monte Carlo Tree Search. Although effective in domains with reliable value signals, such methods depend on precise intermediate feedback and are therefore difficult to apply to open-domain reasoning, where intermediate states are uncertain and lack discrete correctness boundaries.

Unlike the methods described above, our approach avoids computationally intensive online search and instead extracts supervised signals through offline-constructed reasoning trees. Rather than attempting to learn the elusive absolute state value, we derive a “Relative Advantage” signal between Competing Reasoning Prefixes and distill it into a lightweight gating module. This gating module enables efficient, precise control of reasoning trajectories during inference without requiring tree search or rule verification.

\section{Policy Improvement Perspective on Step-wise Reasoning}
We formulate step-wise reasoning control from a policy-improvement perspective. Specifically, the reasoning process of an MLLM is viewed as a sequential decision-making procedure, in which a reasoning gate modifies the underlying policy by selectively suppressing intermediate steps with low expected utility for the final prediction. We first introduce an ideal value-based gating policy and establish its policy-improvement interpretation. We then show how reasoning-tree rollouts provide Monte Carlo supervision for approximating the required action values, thereby connecting the idealized formulation to the learned reasoning gate instantiated in our framework.

\subsection{Step-wise Reasoning as a Finite-Horizon Decision Process}
We model step-by-step reasoning as a finite-horizon decision process. At reasoning step $t$, the state is $h_t = (x, q, \tau_{1: t-1})$, where $x$ is the visual input, $q$ is the question, and $\tau_{1: t-1} = (r_1, \ldots, r_{t-1})$ is the accepted reasoning prefix. The original MLLM samples the next reasoning step according to $\pi_0(r_t \mid h_t)$. A complete trajectory is denoted by $\tau_{1: K}$ and receives a bounded terminal return $R(\tau_{1: K})\in[0, 1]$.

The usefulness of a candidate step can be characterized by its true action value under the original reasoning policy:

\begin{equation}
    Q_t^{\pi_0}(h_t, r_t)=\mathbb{E}_{\tau\sim\pi_0}[R(\tau)\mid h_t, r_t],
\end{equation}

\noindent which measures the expected terminal return after selecting $r_t$ and continuing the subsequent reasoning process with $\pi_0$. The corresponding state value is

\begin{equation}
    V_t^{\pi_0}(h_t)=\mathbb{E}_{r_t\sim \pi_0}[Q_t^{\pi_0}(h_t, r_t)].
\end{equation}

Suppose that the true action value is observable. An oracle value gate retains only reasoning steps whose action values exceed a threshold $T$:

\begin{equation}
    \pi_T^*(r_t\mid h_t)=\frac{\pi_0(r_t\mid h_t)\mathds{1}[Q_t^{\pi_0} (h_t, r_t)\geq T]}{Z_T(h_t)},
\end{equation}

\noindent where $Z_T(h_t)= \Pr_{r_t\sim\pi_0}[Q_t^{\pi_0} (h_t, r_t)\geq T]$ denotes the probability mass assigned by $\pi_0$ to reasoning steps satisfying the gating criterion. Whenever $Z_T(h_t)>0$, the gated policy is well-defined as a valid probability distribution.

\begin{proposition}
\label{prop:gating_trajectory_improvement}
\textbf{Ideal Gating Guarantees Trajectory-Level Performance Monotonicity.}

Let $\pi_0$ be the baseline inference policy and $\pi_T^*$ be the ideal gating policy thresholded at $T$. For any history state $h_t \in \mathcal{H}$ with a non-empty acceptance set, the single-step gated policy satisfies:

\begin{equation}
\mathbb{E}_{r_t \sim \pi_T^*(\cdot \mid h_t)} \left[ Q_t^{\pi_0}(h_t, r_t) \right] \ge \mathbb{E}_{r_t \sim \pi_0(\cdot \mid h_t)} \left[ Q_t^{\pi_0}(h_t, r_t) \right]. \label{eq:single_step_bound_main_text}
\end{equation}

Consequently, under zero intermediate rewards ($R(h_t, r_t) = 0$ for $t < K$), the overall expected return of the gated reasoning policy monotonically dominates that of the baseline policy, \textit{i.e.},

\begin{equation}
    J(\pi^*_T)=\mathbb{E}_{\tau\sim\pi^*_T}[R(\tau)] \geq \mathbb{E}_{\tau\sim\pi_0}[R(\tau)]=J(\pi_0).
\end{equation}

\end{proposition}

\begin{proof}[Proof Sketch]
The proof proceeds in two main steps: establishing local single-step improvement via conditional expectation decomposition, and propagating this advantage across the trajectory via reverse induction.

\textbf{Single-Step Amplification:} Let $S = Q_t^{\pi_0}(h_t, r_t)$ where $r_t \sim \pi_0(\cdot \mid h_t)$, and let $A = \{S \ge T\}$ denote the acceptance event under the threshold $T$. For $\Pr(A) > 0$, by the law of total expectation:

\[
    \mathbb{E}[S] = \Pr (A) \mathbb{E}[S \mid A] + \Pr(A')\mathbb{E}[S \mid A'].
\]

Since every element in $A$ satisfies $S \ge T > S'$ for any $S' \in A'$, it follows that $\mathbb{E}[S | A] \ge \mathbb{E}[S | A']$, which naturally yields $\mathbb{E}[S | A] \ge \mathbb{E}[S]$. Re-substituting $S$ into the conditional expectation $\mathbb{E}[S | A] = \mathbb{E}_{r_t \sim \pi_T^*(\cdot | h_t)} [Q_t^{\pi_0}(h_t, r_t)]$ directly yields Eq.~\eqref{eq:single_step_bound_main_text}.

\textbf{Trajectory-Level Propagation:} Utilizing the dynamic programming operator, the local advantage at step $t$ propagates backwards. Since $V_{t+1}^{\pi_T^*}(h_{t+1}) \ge V_{t+1}^{\pi_0}(h_{t+1})$ has been established for all reachable future states $h_{t+1}$ at step $t+1$, we have:

\begin{equation}
    \begin{aligned}
    V_t^{\pi^*_T}(h_t) &= \mathbb{E}_{r_t \sim \pi^*_T} \left[ \mathbb{E} \left[ V_{t+1}^{\pi^*_T}(h_{t+1}) \;\middle|\; h_t, r_t \right] \right] \\
    &\ge \mathbb{E}_{r_t \sim \pi^*_T} \left[ \mathbb{E} \left[ V_{t+1}^{\pi_0}(h_{t+1}) \;\middle|\; h_t, r_t \right] \right] \\
    &= \mathbb{E}_{r_t \sim \pi^*_T} \left[ Q_t^{\pi_0}(h_t, r_t) \right] \\
    &\ge V_t^{\pi_0}(h_t).
    \end{aligned}
\end{equation}

By induction from the terminal step $K$ down to $t=1$, we obtain $J(\pi_T^*) = V_1^{\pi_T^*}(h_1) \ge V_1^{\pi_0}(h_1) = J(\pi_0)$. A rigorous and complete formal proof is provided in Appendix~\ref{app:proof_gating_monotonicity}.

\end{proof}

\subsection{Reasoning Tree as Monte Carlo Policy Evaluation}
\label{sec:Reasoning_Tree_MC_Theory}
Since the true action-value function $Q_t^{\pi_0}(h_t, r_t)$ defined above cannot be observed directly, we approximate it via Monte Carlo rollouts constructed on a reasoning tree. 

Specifically, given the current state $h_t$ and a sampled action $r_t$, we generate $m$ independent subsequent reasoning completions: $\tau^{(1)}, \dots, \tau^{(m)} \sim \pi_0(\cdot \mid h_t, r_t)$. Each rollout completes the trajectory to termination and receives a final return $R(\tau^{(j)})$. The empirical Monte Carlo estimate of $Q_t^{\pi_0}(h_t, r_t)$ is then computed as:

\begin{equation}
    \widehat{Q}_m(h_t, r_t) = \frac{1}{m}\sum_{j=1}^m R(\tau^{(j)}).
\end{equation}

By the Strong Law of Large Numbers, this estimator is consistent, satisfying $\widehat{Q}_m(h_t, r_t) \xrightarrow{a.s.} Q_t^{\pi_0}(h_t, r_t)$ as $m \rightarrow \infty$. Furthermore, standard concentration inequalities guarantee that the estimation error decays at the standard Monte Carlo rate of $\mathcal{O}(m^{-1/2})$. Specifically, for a target failure probability $\eta \in (0, 1)$ (corresponding to a confidence level $1-\eta$), the estimation error is bounded with probability at least $1-\eta$ by:

\begin{equation}
\begin{aligned}
\left| \widehat{Q}_m(h_t, r_t) - Q_t^{\pi_0}(h_t, r_t) \right| &\leq \epsilon_{\text{MC}}, \\
\text{where} \quad \epsilon_{\text{MC}} &= \sqrt{\frac{\log (2/\eta)}{2m}}.
\end{aligned}
\label{Eq:epsilon_MC_main_text}
\end{equation}

The formal derivation of this error bound via Hoeffding's inequality is deferred to Appendix~\ref{Sec:Monte_Carlo_Policy_Evaluation}.

\subsection{Approximate Policy Improvement with Learned Gates}
In practical deployment, the ideal value-gated policy $\pi_T^*$ cannot be executed directly, as evaluating true action values $Q_t^{\pi_0}(h_t, r_t)$ requires intractable expectation computations. To bridge this gap, we substitute $Q_t^{\pi_0}$ with its empirical Monte Carlo estimator $\widehat{Q}_m(h_t, r_t)$ defined in Eq.~\eqref{Eq:epsilon_MC_main_text}, yielding an empirical gating operator:

\begin{equation}
    \widehat{g}_m(h_t, r_t) = \mathds{1}[\widehat{Q}_m(h_t, r_t)\geq T].
\end{equation}

Furthermore, we fit a parameterized gating module $g_\phi(h_t, r_t) \in [0, 1]$ to approximate $\widehat{g}_m$. The resulting learned reasoning policy is formulated as:

\begin{equation}
\pi_\phi(r_t \mid h_t) = \frac{\pi_0(r_t \mid h_t) g_\phi(h_t, r_t)}{Z_\phi(h_t)},
\end{equation}

\noindent where $Z_\phi(h_t) = \mathbb{E}_{r_t \sim \pi_0}[g_\phi(h_t, r_t)]$ ensures distribution normalization.

The performance gap between the deployed policy $\pi_\phi$ and the ideal oracle policy $\pi_T^*$ originates from two coupled error sources:

\begin{itemize}
    \item Monte Carlo estimation error induced by finite rollout sampling;
    \item Approximation error of the learnable gating operator. 
\end{itemize}

By decomposing the trajectory-level value difference via standard performance difference lemmas in finite-horizon decision processes, the performance degradation of the learned policy $\pi_\phi$ relative to the baseline $\pi_0$ is lower-bounded by:
\begin{equation}
J(\pi_\phi) - J(\pi_0) \ge \Delta^* - \delta_1 K \sqrt{\frac{\log (2/\eta)}{2m}} - \delta_2 K \epsilon_{\text{gate}},
\label{eq:learned_gate_bound}
\end{equation}

\noindent where $\Delta^* = J(\pi_T^*) - J(\pi_0) \ge 0$ represents the theoretical maximum gain from ideal thresholding, $K$ is the reasoning horizon, and $\delta_1, \delta_2 > 0$ are constants depending on reward bounds and policy margins.

Eq.~\eqref{eq:learned_gate_bound} delivers a clear theoretical insight: as the number of Monte Carlo rollouts increases ($m \to \infty$) and the gating model achieves higher classification accuracy ($\epsilon_{\text{gate}} \to 0$), the empirical policy improvement $J(\pi_\phi) - J(\pi_0)$ converges monotonically toward the ideal oracle bound $\Delta^*$. 

\section{Advantage-Guided Long-Horizon Reasoning Framework}
The ideal gated reasoning policy is intractable because the underlying action-value function is unobservable. To approximate it, we construct Monte Carlo reasoning trees to estimate trajectory returns, train learnable gating modules from sampled trajectories, and perform step-wise gated decoding during inference. As illustrated in Fig.~\ref{fig: Monte_Carlo_Reasoning_Tree_Construction}, we construct Monte Carlo reasoning trees and derive hierarchical supervision through terminal-answer evaluation and bottom-up value propagation.

\begin{figure*}[t]
    \centering
    \includegraphics[width=\linewidth]{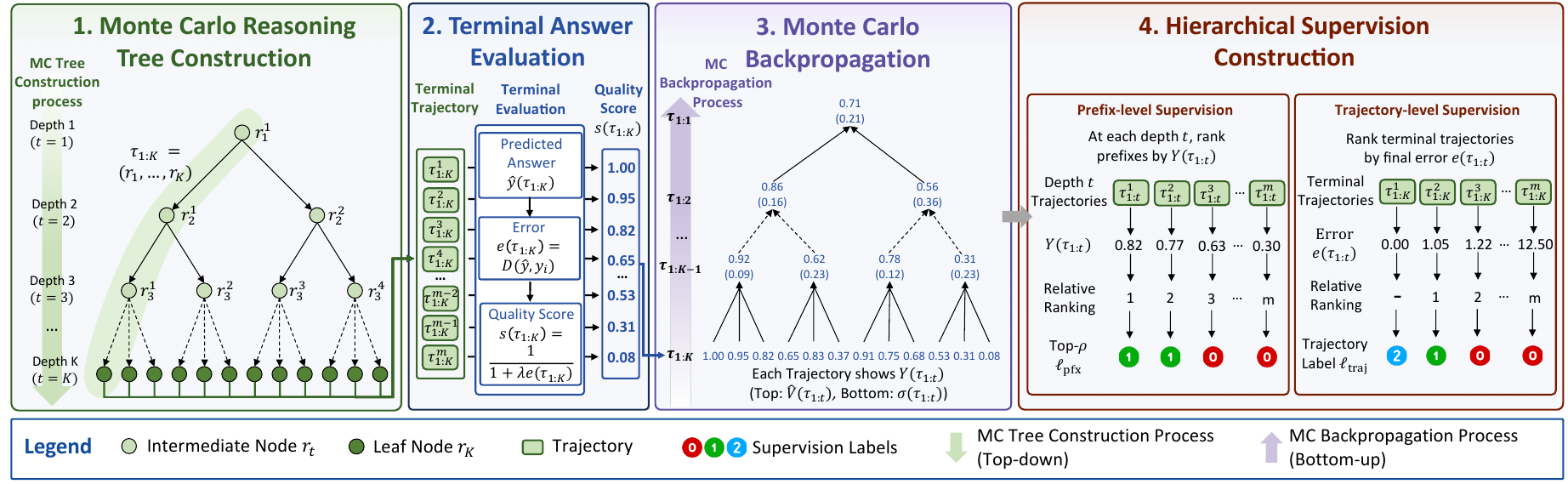}
    \caption{Construction and annotation of the Monte Carlo reasoning tree. The reasoning tree is first expanded top-down by recursively sampling candidate reasoning steps until terminal trajectories are obtained. Each terminal answer is then evaluated against the ground truth and converted into a trajectory-quality score. These terminal scores are propagated bottom-up to estimate the values of intermediate prefixes. Finally, hierarchical supervision is constructed at two levels: prefixes at each depth are ranked by their Monte Carlo estimates to generate Top-\(\rho\) accept/reject labels for SAG, while terminal trajectories are ranked by answer error to generate trajectory-level labels for TAG.}
    \label{fig: Monte_Carlo_Reasoning_Tree_Construction}
    \vspace{-0.5cm}
\end{figure*}

\subsection{Monte Carlo Reasoning Tree Construction}
To obtain the empirical value estimates required for gating supervision, we explicitly formulate the step-by-step generation process as a multi-branch reasoning tree. For a given visual question-answering instance $\mathcal{I}_i = (x_i, q_i, y_i)$ with ground-truth answer $y_i$, we construct a corresponding reasoning tree $\mathcal{T}_i = (\mathcal{V}_i, \mathcal{E}_i)$.

Each node in $\mathcal{V}_i$ corresponds to a specific reasoning state $h_t = (x_i, q_i, \tau_{1:t-1})$. The root node initializes the process with an empty prefix. To expand the tree top-down, we recursively sample candidate reasoning steps from the baseline MLLM policy $\pi_0(\cdot \mid h_t)$. This multi-branch generation continues until a termination signal is produced, forming a leaf node.

The set of all leaf nodes constitutes the terminal trajectories $\mathcal{V}_i^{\text{traj}}$. Each complete trajectory $\tau_{1:K} \in \mathcal{V}_i^{\text{traj}}$ yields a final predicted answer $\hat{y}(\tau_{1:K})$, which is evaluated against the ground-truth $y_i$ to produce a deterministic trajectory-quality score (\textit{i.e.}, the terminal return $R(\tau_{1:K})$).

\subsubsection{Backward Propagation of Terminal Feedback on the Reasoning Tree}
After obtaining the terminal quality of the terminal trajectory, we propagate the terminal quality upward along the tree structure to estimate the oracle value of each intermediate reasoning prefix. For any intermediate prefix $\tau_{1:t}$, let $L(\tau_{1:t})$ denote the set of all reachable leaf nodes in the sampling subtree rooted at prefix $\tau_{1:t}$. The value score of prefix $\tau_{1:t}$ is defined as:

\begin{equation}
    \widehat{V}(\tau_{1:t})
=
\frac{1}{|L(\tau_{1:t})|}\sum_{\tau_{1:K}\in L(\tau_{1:t})} s(\tau_{1:K}).
\end{equation}

It should be noted that $\widehat{V}(\tau_{1:t})$ is not used to determine whether the current reasoning step itself is locally correct, but rather to characterize the long-term impact of the current reasoning state on the ultimate success of the task. Therefore, $\widehat{V}(\tau_{1:t})$ can be viewed as a finite-sample Monte Carlo estimate of the oracle action-value, with supervision derived from statistical feedback over reachable terminal trajectories.

Compared to directly labeling intermediate reasoning steps for local correctness, this long-horizon value estimation based on final result feedback is better suited for open-ended reasoning scenarios. This is because the local validity of intermediate reasoning steps is often difficult to assess independently, whereas their true value typically lies in whether subsequent reasoning continues to lead to high-quality final answers.

Furthermore, to characterize the stability among different subsequent branches under the same reasoning prefix, we introduce uncertainty in the subtree quality distribution:

\[
    \sigma(\tau_{1:t}) = \operatorname{Std}\{s(\tau_{1:K})\mid \tau_{1:K}\in L(\tau_{1:t})\}.
\]

Based on this, we define the exploration-aware value score:

\[
    Y(\tau_{1:t})=\widehat{V}(\tau_{1:t})+\alpha \sigma(\tau_{1:t}),
\]

\noindent where $\alpha$ controls the weight of the uncertainty term, this formulation can be interpreted as an optimistic value estimation, where the uncertainty term characterizes the potential exploratory benefits of the current reasoning state. This design simultaneously encapsulates the expected final quality of the reasoning trajectory and accounts for the potential diversity and optimization space of its subsequent reasoning branches. Compared to value estimation relying exclusively on mean statistics, the introduction of the uncertainty term enables the system to preserve certain reasoning states with high potential utility but substantial branch volatility, thereby mitigating the risk of premature search space constriction and insufficient exploration during open-ended reasoning.

\subsubsection{Hierarchical Supervision Construction}
Based on the exploration-aware value score $Y(\tau_{1:t})$, we further construct hierarchical supervision signals: prefix-level supervision for intermediate reasoning prefixes and trajectory-level supervision for complete reasoning trajectories. Since the scale of value scores varies across different problems, reasoning trees, and depths, we do not directly apply a fixed global threshold. Instead, we perform relative ranking within each tree to align the supervision distributions across different samples.

\textbf{Prefix-level Supervision Construction.}
For a set of intermediate nodes at the same depth in a reasoning tree:
\[
    \mathcal{V}_{i,t}=\{\tau_{1:t}\in \mathcal{V}_i\mid d(\tau_{1:t})=t\}.
\]

\noindent where $d(\cdot)$ represents the depth.

We rank the prefixes based on their exploration-aware value scores $Y(\tau_{1:t})$ and define the prefixes with higher relative rankings as high-value reasoning prefixes:

\begin{equation}
\ell_{1:t}=
\begin{cases}
1, & \tau_{1:t}\in \operatorname{Top}_\rho(\mathcal{V}_{i,t};Y),\\
0, & \text{otherwise}.
\end{cases}
\end{equation}

\noindent where $\ell_{1:t} \in \{0, 1\}$ represents the supervised label of a node, and $\operatorname{Top}_\rho(\cdot)$ denotes the set of the top $\rho$ nodes ranked by $Y(\tau_{1:t})$. This relative-ranking-based supervision scheme offers two advantages. First, it naturally aligns supervision signals across different samples and tasks by relying on relative ordering rather than absolute value scales. Second, it learns the relative advantage among candidate reasoning prefixes under similar reasoning contexts instead of estimating precise values. As a result, the learning objective becomes more stable and easier to optimize, which is particularly beneficial for open-ended reasoning with noisy terminal rewards.

\textbf{Trajectory-level Supervision Construction.}
In addition to the supervision signal for intermediate reasoning prefixes, we further construct trajectory-level supervision for the complete reasoning trajectories corresponding to leaf nodes.

For any terminal trajectory $\tau_{1:K}$, its supervision label is defined as:

\[
    \ell_{1:K} = \mathcal{A}(e(\tau_{1:K}), \operatorname{rank}(\tau_{1:K})),
\]

\noindent where $e(\tau_{1:K})$ represents the error of the final answer, $\operatorname{rank}(\tau_{1:K})$ denotes the relative ranking based on error within the current reasoning tree, and $\mathcal{A}(\cdot)$ denotes the label mapping function. Specifically, trajectory-level labels account for both the correctness of the final answer and its relative quality among candidate trajectories within the tree. Completely correct trajectories are assigned the highest-quality label; for the remaining trajectories, quality is tiered based on their relative error rankings. Consequently, trajectory-level supervision not only captures whether the final answer is correct but also preserves the quality ranking information among different candidate trajectories, thereby forming a more discriminative trajectory-level supervision. For notation simplicity, we denote $\ell_{1:t}$ and $\ell_{1:K}$ as $\ell_\mathrm{pfx}$ and $\ell_\mathrm{traj}$, respectively.

Following the above process, each reasoning tree ultimately generates two types of hierarchical supervision data.

The first type is prefix-level supervision corresponding to intermediate reasoning prefixes:

\[
    \mathcal{D}_{\mathrm{prefix}}=\{(x_i, q_i, \tau_{1:t}, \ell_\mathrm{pfx})\mid \tau_{1:t}\in \mathcal{V}_i \backslash \mathcal{V}_i^{\text{traj}}\}.
\]

The second type is trajectory-level supervision corresponding to complete reasoning trajectories:

\[
\mathcal{D}_{\mathrm{traj}}
=
\{(x_i,q_i,\tau_{1:K},\hat{y}(\tau_{1:K}), \ell_\mathrm{traj})
\mid \tau_{1:K}\in\mathcal{V}_i^{\mathrm{traj}}\}.
\]

Consequently, the original open-ended, multi-branch, and non-deterministic reasoning generation process is transformed into a hierarchical, trajectory-aware supervision framework. Compared to supervision that relies solely on final answers, this approach leverages information from both intermediate reasoning states and final reasoning trajectories, expanding sparse final feedback into a hierarchical supervision signal that spans the entire reasoning trajectory. This provides a more fine-grained learning basis for modeling the quality of long-horizon reasoning.

\begin{figure*}[t]
    \centering
    \includegraphics[width=\linewidth]{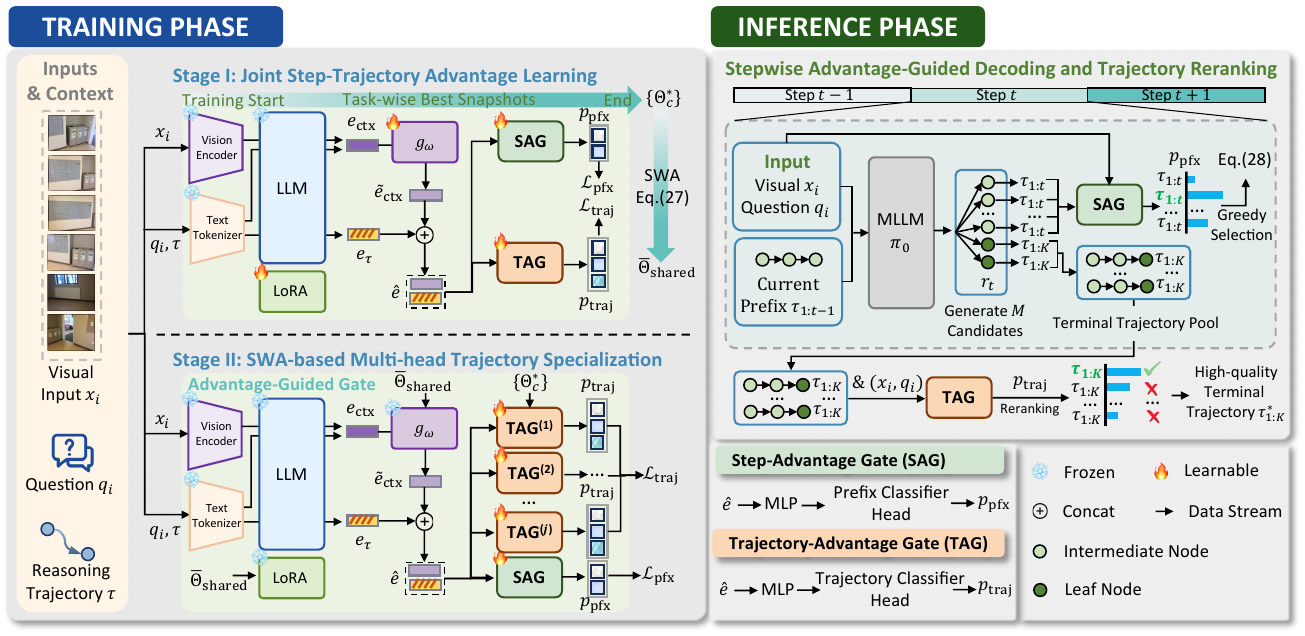}
    \caption{Overview of the proposed Advantage-Guided Gate framework. The training phase consists of two stages. Stage~I jointly optimizes the Step-Advantage Gate (SAG) and Trajectory-Advantage Gate (TAG) using prefix- and trajectory-level supervision, while retaining task-wise best checkpoints. Their shared parameters are aggregated through stochastic weight averaging (SWA) to initialize Stage~II, where the shared representation and SAG are preserved and multiple task-specialized TAG heads are further optimized. During inference, the frozen MLLM generates multiple candidates at each reasoning step, SAG greedily retains the most promising prefix for continued expansion, and TAG reranks the resulting terminal trajectories to select the final high-quality solution. Snowflakes and flames denote frozen and learnable modules, respectively.}
    \label{fig: overview}
    \vspace{-0.5cm}
\end{figure*}

\subsection{Advantage-Guided Gated Learning}
After constructing the reasoning tree, we further train an advantage-guided gating model to filter intermediate reasoning prefixes during the reasoning process and select the most reliable output from multiple complete reasoning answers at the final stage. As illustrated in Fig.~\ref{fig: overview}, our Advantage-Guided Gate framework integrates two-stage advantage learning with stepwise gated decoding and trajectory reranking.

Given a sample $(x_i, q_i)$ and its reasoning prefix $\tau_{1:t}$, we model the reasoning path representation and the visual-problem context representation separately. Given the current reasoning prefix $\tau_{1:t}$, we use an MLLM-based reasoning encoder to extract its dynamic representation:

\[
    \mathbf{e}_{\tau_{1:t}} = f_\theta(\tau_{1:t})
\]

\noindent where $\mathbf{e}_{\tau_{1:t}}$ is used to characterize the semantic reasoning state corresponding to the current reasoning prefix.

Meanwhile, since the visual input and the semantic meaning of the question remain consistent across different reasoning prefixes within the same sample, we further construct static visual-question context representations:

\[
    \mathbf{e}_\mathrm{ctx}=f_{\text{ctx}}(x_i, q_i).
\]

Next, the context representation is mapped to a feature space compatible with the reasoning representation using a lightweight context projection module:

\[
    \tilde{\mathbf{e}}_\mathrm{ctx}=g_\omega(\mathbf{e}_\mathrm{ctx}),
\]

\noindent where $g_\omega$ denotes the context projection module.

Finally, we concatenate the reasoning prefix representation with the projected context representation:

\[
    \hat{\mathbf{e}}=[\mathbf{e}_{\tau_{1:t}}; \tilde{\mathbf{e}}_\mathrm{ctx}].
\]

The fused representation is then passed to prediction branches that estimate the latent value of the intermediate prefix and the final quality of the complete trajectory.

\subsubsection{Stage I: Joint Step-Trajectory Advantage Learning}
Based on the fused representation $\hat{\mathbf{e}}$, we further construct two distinct predictive branches, corresponding to the \textbf{Step-Advantage Gate (SAG)} and the \textbf{Trajectory-Advantage Gate (TAG)}, respectively. Specifically, the Step-Advantage Gate models the long-horizon latent advantage of intermediate reasoning states, whereas the Trajectory-Advantage Gate quantifies the terminal-quality advantage associated with the complete reasoning trajectory.

\textbf{Step-Advantage Gate.} For any intermediate reasoning prefix $\tau_{1:t}$, the Step-Advantage Gate outputs its advantage probability:

\[
    p_\mathrm{pfx}(\ell_\mathrm{pfx})=P_{\phi_\text{SAG}}(\ell_\mathrm{pfx}=1\mid x_i, q_i, \tau_{1:t}),
\]

\noindent where $\ell_\mathrm{pfx}$ denotes the prefix-level supervision label obtained based on the relative advantage ranking within the tree.

Since high-advantage reasoning prefixes typically constitute only a small fraction of the open-ended reasoning space, prefix-level supervision inherently suffers from class distribution imbalance. Therefore, we adopt a weighted binary cross-entropy objective:

\[
    \mathcal{L}_\mathrm{pfx} = \operatorname{WBCE}(p_\mathrm{pfx}(\tau_{1:t}), \ell_\mathrm{pfx}),
\]

\noindent where positive samples correspond to highly advantageous reasoning prefixes with higher exploration-aware value scores $Y(\tau_{1:t})$ in intra-layer reasoning competition, while negative samples correspond to relatively less advantageous prefixes. This supervised approach does not learn absolute reasoning values, but rather learns relative advantage relationships among different reasoning prefixes.

\textbf{Trajectory-Advantage Gate.}
For a complete reasoning trajectory $\tau_{1:K}$ and its final answer $\hat{y}(\tau_{1:K})$, the Trajectory-Advantage Gate outputs a trajectory-level quality distribution:

\[
    p_\mathrm{traj}(\ell_\mathrm{traj})=P_{\phi_\text{TAG}}(\ell_\mathrm{traj}\mid x_i, q_i, \tau_{1:K}, \hat{y}(\tau_{1:K})),
\]

\noindent where $\ell_\mathrm{traj}$ denotes the trajectory-level supervision label.

Similar to the Step-Advantage Gate, since high-advantage trajectories typically constitute only a minority in an open-ended reasoning space, trajectory-level supervision also exhibits an imbalanced distribution. Therefore, we adopt a reweighted cross-entropy objective:

\begin{equation}
    \begin{split}
    \mathcal{L}_{\mathrm{traj}}&
= \\
&-
\sum_{\tau_{1:K}\in\mathcal{D}_{\mathrm{traj}}}
w_{\ell_\mathrm{traj}}
\log
P_{\phi_\text{TAG}}
(
\ell_\mathrm{traj}
\mid
x_i,q_i,\tau_{1:K},\hat{y}(\tau_{1:K})
).
    \end{split}
\end{equation}

\noindent where

\[
    w_{\ell_\mathrm{traj}}
\propto
\frac{1}{N_{\ell_\mathrm{traj}}}.
\]

\noindent where $N_{\ell_\mathrm{traj}}$ denotes the number of training samples belonging to the trajectory label $\ell_\mathrm{traj}$. This reweighting strategy mitigates the bias in the supervision distribution across different trajectory advantage levels, thereby preventing the model from becoming overly biased toward high-frequency, low-advantage trajectories during training.

The joint optimization objective in Stage I is defined as

\[
\mathcal{L}^{(1)}= \mathcal{L}_{\mathrm{pfx}}+\mathcal{L}_{\mathrm{traj}}.
\]

During this stage, the MLLM backbone remains frozen, while only the parameter-efficient adaptation modules, the context projection module $g_{\omega}$, the Step-Advantage Gate, and the Trajectory-Advantage Gate are updated.
By jointly optimizing the prefix-level and trajectory-level supervision signals, Stage I learns a shared long-horizon reasoning representation that captures both the potential value of intermediate reasoning prefixes and the quality of complete reasoning trajectories.

\subsubsection{Task-wise Optimal State Selection}
Since different task types exhibit significant differences in reasoning patterns and trajectory-level answer preferences, they typically achieve optimal performance at different stages of the training process. This suggests that, although the first stage focuses on learning shared long-horizon reasoning representations, different tasks still achieve optimal generalization capabilities at different points along the optimization trajectory.

Based on this observation, during the first training phase, the historical best model snapshots that performed best on the validation set for each task type $c$ will be saved separately:

\[
    \Theta_c^*
=
\arg\max \operatorname{Eval}_c(\Theta^{(1)}),
\quad c\in\mathcal{C}.
\]

\noindent where $\mathcal{C}=\{1,\ldots,C\}$, and $\Theta^{(1)}$ denotes the Stage I optimization trajectory comprising all model states recorded at validation snapshots.

This yields the task-wise best snapshots:

\[
    \mathcal{B} = \{\Theta_c^*\}_{c=1}^C.
\]

It is important to note that these snapshots do not correspond to multiple independent models, but rather represent different optimization states along the same shared optimization trajectory. They reflect the task-wise optimal states achieved when optimal generalization performance is attained for different task types, and will serve as the initialization source for subsequent trajectory specialization.

\subsubsection{Stage II: SWA-based Multi-head Trajectory Specialization}
Given the collection of task-wise optimal states obtained in Stage I, we further perform trajectory-level specialization in Stage II. We first perform stochastic weight averaging (SWA) on the shared parameters of the task-wise optimal states:

\begin{equation}
    \bar{\Theta}_\mathrm{shared}=\frac{1}{C}\sum_{c=1}^C \Theta_{\mathrm{shared}, c}^{*}.
\end{equation}
    
\noindent where these shared parameters include efficient adaptation parameters and context projection modules.

By aggregating the shared parameters of multiple task-wise optimal states, SWA constructs a parameter initialization shared across tasks. This effectively mitigates the optimization bias introduced by single-task optimal states and provides a more stable shared representation space for subsequent trajectory-level specialization. After SWA is completed, the shared parameters are frozen, leaving the Trajectory-Advantage Gate and the Step-Advantage Gate to participate in subsequent optimization.

Building on this foundation, we further employ a multi-head trajectory-specialized architecture, where each head has the same structure but is initialized with different task-wise optimal states.

For the $j$-th trajectory head, its output is defined as:

\[
    p_\mathrm{traj}^{(j)}(\tau_{1:K})=P_{\phi_\text{TAG}^{(j)}}(\ell_\mathrm{traj}\mid x_i, q_i, \tau_{1:K}, \hat{y}(\tau_{1:K})).
\]

\noindent where

\[
    \Theta_c^{*} \rightarrow \phi_\text{TAG}^{(j)},
\]

\noindent indicates that the $j$-th trajectory head is initialized using the Trajectory-Advantage Gate parameters extracted from the corresponding task-wise optimal state.

Therefore, in Stage II,  the parameters $\phi_\text{TAG}^{(j)}$ of all Trajectory-Advantage heads and the Step-Advantage Gate are updated. At this stage, the optimization focus shifts from relearning shared long-horizon reasoning representations to a more structured exploration of the trajectory-level answer-selection space.

At each validation cycle, all trajectory heads are evaluated across task types, and the best historical head and snapshot for each task are updated:
\[
j_c^* 
=
\arg \max_{j} 
\operatorname{Eval}_{c}(\Theta_{j}^{(2)}).
\]

\noindent where $j_c^*$ denotes the index of the trajectory-specialized head that achieves the best validation performance for task type $c$.

It is important to note that the multi-head architecture does not imply a one-to-one correspondence between different heads and specific task types. Instead, all trajectory heads share the full training dataset and are independently optimized starting from different task-wise optimal states. Therefore, the different heads can be viewed as multiple optimization explorations of the trajectory-level answer selection space, thereby enhancing the robustness and generalization ability of the final trajectory selection.

\subsection{Advantage-Guided Decoding and Trajectory Reranking}
During inference, we perform stepwise candidate-level advantage selection. Given the current reasoning prefix $\tau_{1:t}$, the model samples $M$ next-step candidates from the original generation policy:

\[
    \{r_{t+1}^{(m)}\}_{m=1}^M.
\]

Each candidate is concatenated with the current reasoning prefix to form a new candidate prefix:

\[
    \tau_{1:t+1}^{(m)}=(\tau_{1:t}, r_{t+1}^{(m)}).
\]


For candidate prefixes that have generated a final answer, their corresponding complete reasoning trajectories $\tau_{1: t+1}^{(m)}$ and answers $\hat{y}^{(m)}$ are added to the trajectory candidate pool for subsequent trajectory-level quality evaluation. For candidate prefixes that have not yet terminated, they are further fed into the Step-Advantage Gate together with the visual input $x_i$ and the question $q_i$ to evaluate their long-horizon reasoning potential:

\[
    p_\mathrm{pfx}^{(m)} = P_{\phi_\text{SAG}}(\ell_\mathrm{pfx}=1\mid x_i, q_i, \tau_{1:t+1}^{(m)}).
\]

Ultimately, we adopted a greedy policy, selecting the candidate with the highest probability of success as the prefix for the next reasoning step:

\begin{equation}
    \begin{split}
        & m^* = \arg \max_{m\in\{1, \dots, M\}} p_\mathrm{pfx}^{(m)}, \\
& \tau_{1:t+1}=\tau_{1:t+1}^{(m^*)}.
    \end{split}
\end{equation}

This process iterates step by step until the maximum inference depth is reached or a termination token is generated.


For the trajectory candidates ultimately generated in the trajectory candidate pool, the model further employs the Trajectory-Advantage Gate to evaluate the final quality of the fully reasoned trajectories. We input $(x_i, q_i, \tau_{1:K}, \hat{y}(\tau_{1:K}))$ into the optimal Trajectory-Advantage head for the corresponding task type and compute the probability that it belongs to the highest-quality class:

\[
    p_{\mathrm{traj}} = P_{\phi_\text{TAG}}(\ell_\mathrm{traj}=\ell_{\max} \mid x_i, q_i, \tau_{1:K}, \hat{y}(\tau_{1:K})).
\]

The candidate with the highest probability of achieving a high-quality terminal result is selected as the final output:

\begin{equation}
    \tau_{1:K}^* = \arg \max_{\tau_{1:K}}p_\mathrm{traj}.
\end{equation}

Overall, the inference stage does not rely on fixed thresholds to independently accept or reject individual candidates, but rather compares the relative advantages of multiple candidates within the same inference state. The Step-Advantage Gate is responsible for progressively selecting reasoning prefixes with greater long-horizon potential, while the Trajectory-Advantage Gate is responsible for selecting the answer with the highest final-state quality from the complete candidate trajectories.

From a theoretical perspective, this reasoning process corresponds to a stepwise greedy approximation of the advantage policy. According to Equ.~\ref{Eq:oracle_error}, the deviation of the learned policy from the ideal advantage gain $\Delta^*$ primarily stems from Monte Carlo sampling error and gate approximation error. Therefore, the two-stage advantage-guided optimization framework proposed in this paper aims further to reduce $\epsilon_\text{gate}$ through hierarchical advantage supervision and trajectory-level specialization, thereby improving the overall return of the long-horizon reasoning policy.
\section{Experiment}

\subsection{Experimental Settings}

We evaluate the proposed advantage-guided gating method across four 3D scene datasets: ScanNet~\cite{dai2017scannet}, ScanNet++~\cite{yeshwanth2023scannet}, Matterport3D~\cite{chang2017matterport3d}, and HM3D~\cite{ramakrishnan2habitat}. Based on the ground-truth 3D annotations for each dataset, we construct a spatial question-answering task and generate counterfactual questions by combining different question types to reduce the likelihood that the model will rely on perception shortcuts to answer directly. Following the settings of VSI-Bench~\cite{yang2025thinking} and Spatial-MLLM~\cite{wu2026spatial}, the task covers six categories of spatial questions: object counting, absolute distance estimation, relative distance judgment, object size estimation, room size estimation, and relative direction judgment. The visual input for each question consists of multi-frame images of the corresponding scene or room region. ScanNet and ScanNet++ use video keyframes; Matterport3D uses NFOV multi-view images aligned with the region; and HM3D employs room-level multi-view images obtained through coverage-aware camera sampling. Additionally, we constructed the Reasoning-Tree-160K dataset for training gated models. This dataset contains multi-branch reasoning trees along with step-level and trajectory-level supervision signals, used to learn filtering of intermediate reasoning steps and evaluation of final answers.

\noindent \textbf{Evaluation Metrics.} For evaluation metrics, we follow the settings specified by VSI-Bench. For multiple-choice questions, we use accuracy: a score of $1$ is assigned when the predicted option matches the correct answer, and $0$ otherwise. For numerical questions, we use Mean Relative Accuracy (MRA) to evaluate the predictive performance. MRA calculates the average of the relative accuracy across a set of confidence thresholds $\bar{\mathcal{C}}=\{0.50,0.55,\ldots,0.95\}$:
\[
\mathrm{MRA}
\frac{1}{|\bar{\mathcal{C}}|}
\sum_{\bar{\theta}\in\bar{\mathcal{C}}}
\mathbb{I}
\left(
\frac{|\hat{y}-y|}{y}<1-\bar{\theta}
\right),
\]
where $\hat{y}$ and $y$ denote the predicted and true values, respectively, and $\mathbb{I}(\cdot)$ is the indicator function. We report the average scores for each question type separately, as well as the overall average across all question types. Following Spatial-MLLM~\cite{wu2026spatial}, we report Avg. as the micro-averaged score over all evaluation samples.

\noindent \textbf{Implementation Details.} Our gating model employs a two-stage training process, using Qwen3-VL-4B~\cite{bai2025qwen3} as the value estimation backbone network, on which a lightweight gating module is trained. The batch sizes for the two stages are set to 16 and 32, respectively; the gradient accumulation steps are set to 4; the learning rate is set to $2 \times 10^{-4}$; and the number of training iterations is 3 for both stages. In the first stage, we jointly train the shared projection module, LoRA parameters, Step-Advantage Gate (SAG), and Trajectory-Advantage Gate (TAG) after training the vision-language backbone network to estimate the values of the prefix node and the terminal answer node, respectively. We split the training and validation sets in a $9:1$ ratio and select the optimal parameters based on validation set performance. In the second stage, we initialize the model based on the optimal parameters from the first stage: we perform stochastic weighted averaging (SWA) on the LoRA parameters and the shared projection MLP, and initialize independent TAGs for different question types. During training, we freeze the vision-language backbone network, LoRA parameters, and the shared projection MLP, updating only the SAG and the TAGs for each question type. During the testing phase, the prefix node is scored by SAG, while the terminal answer node is scored by inputting the corresponding TAG based on the question type. Based on the ablation results on the number of candidates, each inference step generates 20 next-step candidates by default, with a maximum inference depth of $K=6$.

\begin{table*}[!t]
\centering
\caption{Quantitative comparison on ScanNet. Numerical and multiple-choice questions are evaluated by MRA and accuracy, respectively. Avg. is used for ranking. \textbf{Bold} denotes the best result in each column. Green rows show results with Advantage-Guided Gate, with colored values indicating absolute changes from the corresponding Stepwise baseline. SI denotes spatial instruction tuning. All scores are percentages.}
\resizebox{15.5cm}{!}{
\begin{tabular}{l|cccc|cc|cc}
\toprule 
\multirow{2}{*}{Methods} 
& \multicolumn{4}{c|}{Numerical} 
& \multicolumn{2}{c|}{Multiple-Choice} 
& \multirow{2}{*}{Avg.} 
& \multirow{2}{*}{Rank} \\
\cline{2-7}
& Obj. Cnt. & Abs. Dist. & Obj. Size & Room Size 
& Rel. Dist. & Rel. Dir.
& & \\
\hline
\rowcolor{myblue}
\textit{\textbf{Proprietary Models (API)}} & & & & & & & & \\
Gemini-2.5-pro & 20.4 & 29.2 & 27.8 & 13.9 & 26.8 & 38.0 & 28.9 & 22 \\
Gemini-3.1-pro-preview (CoT) & 37.4 & 23.2 & 42.6 & 0.0 & 46.3 & 53.3 & 40.8 & 4 \\
GPT-5.4 (CoT) & 35.6 & 38.0 & 33.8 & 1.4 & 24.4 & 44.0 & 35.3 & 7 \\
GPT-5.4 (Self-Consistency-20) & 28.2 & 40.8 & 33.4 & 9.3 & 29.3 & 49.3 & 36.5 & 5 \\
ToT (GPT-5.4) & 33.0 & \textbf{52.2} & 35.4 & 6.1 & 39.0 & 48.7 & 40.9 & 3 \\
\hline
\rowcolor{myblue}
\textit{\textbf{Open-source General Models}} & & & & & & & & \\
MiMo-VL-7B-SFT~\cite{li2025xiaomi} & 29.8 & 35.4 & 29.6 & 0.0 & 17.1 & 43.3 & 32.4 & 13 \\
MiMo-VL-7B-RL~\cite{li2025xiaomi} & 25.0 & 20.2 & 34.2 & 2.5 & 24.4 & 42.0 & 30.7 & 20 \\
LLaMA-3.2V-11B~\cite{grattafiori2024llama} & 27.2 & 22.8 & 23.2 & 17.5 & 31.7 & 39.3 & 30.8 & 19 \\
LLaMA-3.2V-11B-CoT~\cite{xu2025llava} & 30.8 & 23.4 & 23.6 & 15.4 & 26.8 & 40.7 & 31.2 & 18 \\
Kimi-VL-16B-A3B~\cite{team2025kimi} & 14.6 & 8.0 & 19.4 & 34.3 & 31.7 & 37.3 & 27.0 & 23 \\
Kimi-VL-16B-A3B-Thinking~\cite{team2025kimi} & 17.0 & 25.2 & 30.0 & 4.3 & 26.8 & 33.3 & 26.6 & 24 \\
Deepseek-VL2-27B-A4.5~\cite{wu2024deepseek} & 35.0 & 16.0 & 35.6 & 11.8 & 31.7 & 37.3 & 31.3 & 17 \\
Qwen3-VL-8B~\cite{bai2025qwen3} & 10.6 & 23.2 & 10.2 & 2.9 & 24.4 & 40.7 & 25.4 & 26 \\
Qwen3-VL-32B~\cite{bai2025qwen3} & 19.8 & 34.4 & 39.8 & 2.9 & 24.4 & 42.7 & 33.0 & 11 \\
LLaVA-NeXT-Video-34B~\cite{zhang2410video} & 10.0 & 17.8 & 38.0 & 1.1 & 19.5 & 27.3 & 22.3 & 28 \\
InternVL3-78B~\cite{zhu2025internvl3} & 24.4 & 40.2 & 37.0 & 39.6 & 19.5 & 32.7 & 32.2 & 14 \\
\hline
\rowcolor{myblue}
\textit{\textbf{Open-source SI Models}} & & & & & & & & \\
SpaceLLaVA-13B~\cite{chen2024spatialvlm} & 26.2 & 21.8 & 25.2 & 3.6 & 26.8 & 30.7 & 25.6 & 25 \\
SpaceR-7B~\cite{ouyang2025spacer} & 9.6 & 35.6 & 29.2 & 17.1 & 24.4 & 36.7 & 29.0 & 21 \\
SenseNova-SI-1.1-Qwen2.5-VL-7B~\cite{cai2026scaling} & 23.2 & 41.6 & 32.0 & 18.6 & 26.8 & 37.3 & 32.7 & 12 \\
SenseNova-SI-1.1-Qwen3-VL-8B~\cite{cai2026scaling} & 2.2 & 34.0 & 14.4 & 10.7 & 14.7 & 38.0 & 24.7 & 27 \\
SenseNova-SI-1.3-InternVL3-8B~\cite{cai2026scaling} & 32.0 & 27.6 & 30.2 & 42.5 & 31.7 & 37.3 & 34.1 & 9 \\
VST-7B-SFT~\cite{yang2025visual} & 27.8 & 23.0 & 43.8 & 15.7 & 24.4 & 36.0 & 31.4 & 16 \\
Spatial-MLLM-v1.1-135K~\cite{wu2026spatial} & 32.6 & 40.4 & 30.4 & 22.5 & 17.1 & 35.3 & 32.0 & 15 \\
Spatial-MLLM-v1.1-820K~\cite{wu2026spatial} & 44.2 & 31.2 & 33.0 & \textbf{53.6} & 22.0 & 34.0 & 35.0 & 8 \\
\hline
GPT-5.4 (Stepwise) & 36.8 & 47.2 & 38.4 & 16.1 & 29.3 & 31.3 & 34.1 & 9 \\
\rowcolor{mygreen}
\textbf{+ Advantage-Guided Gate}
& \textbf{46.6}\upgain{9.8}
& 49.8\upgain{2.6}
& 40.0\upgain{1.6}
& 29.3\upgain{13.2}
& 46.3\upgain{17.0}
& 46.0\upgain{14.7}
& 44.6\upgain{10.5}
& 2 \\
Gemini-3.1-pro-preview (Stepwise) & 41.8 & 33.8 & 40.6 & 15.0 & 39.0 & 36.7 & 36.1 & 6 \\
\rowcolor{mygreen}
\textbf{+ Advantage-Guided Gate}
& 45.2\upgain{3.4}
& 30.8\downdrop{3.0}
& \textbf{45.4}\upgain{4.8}
& 47.1\upgain{32.1}
& \textbf{51.2}\upgain{12.2}
& \textbf{57.3}\upgain{20.6}
& \textbf{49.0}\upgain{14.0}
& 1 \\
\bottomrule
\end{tabular}
}
\label{tab:Main_Results}
\vspace{-0.3cm}
\end{table*}

\subsection{Main Results and Analysis}
Tables ~\ref{tab:Main_Results}–\ref{tab:HM3D} present the main experimental results of the proposed Advantage-Guided Gate method on four 3D scene benchmarks. Overall, the method achieves consistent performance improvements across different base models and datasets. For example, on ScanNet, Advantage-Guided Gate boosts the overall scores of GPT-5.4~\cite{singh2025openai} and Gemini-3.1-pro-preview~\cite{googledeepmind2026gemini31pro} from 34.1 and 36.1 to 44.6 and 49.0, respectively; on ScanNet++, GPT-5.4’s overall score increased from 28.7 to 43.8, representing a gain of 15.1 percentage points. Similar improvement trends were also observed on Matterport3D and HM3D, indicating that the proposed method is not limited to specific base models or data distributions but can adapt to different scene structures, visual observation patterns, and spatial distribution characteristics.

\begin{table*}[!t]
\centering
\caption{Quantitative comparison on ScanNet++. Numerical and multiple-choice questions are evaluated by MRA and accuracy, respectively. Avg. is used for ranking. \textbf{Bold} denotes the best result in each column. Green rows show results with Advantage-Guided Gate, with colored values indicating absolute changes from the corresponding Stepwise baseline. SI denotes spatial instruction tuning. All scores are percentages.}
\resizebox{15.5cm}{!}{
\begin{tabular}{l|cccc|cc|cc}
\toprule 
\multirow{2}{*}{Methods} 
& \multicolumn{4}{c|}{Numerical} 
& \multicolumn{2}{c|}{Multiple-Choice} 
& \multirow{2}{*}{Avg.} 
& \multirow{2}{*}{Rank} \\
\cline{2-7}
& Obj. Cnt. & Abs. Dist. & Obj. Size & Room Size 
& Rel. Dist. & Rel. Dir.
& & \\
\hline
\rowcolor{myblue}
\textit{\textbf{Proprietary Models (API)}} & & & & & & & & \\
Gemini-2.5-pro & 20.6 & 21.0 & 38.4 & 11.4 & 30.0 & 32.0 & 27.2 & 17 \\
Gemini-3.1-pro-preview (CoT) & 26.4 & 31.4 & 31.4 & 3.8 & 30.0 & 45.3 & 32.4 & 4 \\
GPT-5.4 (CoT) & 30.4 & 39.6 & 29.2 & 6.0 & 28.0 & 36.7 & 30.4 & 8 \\
GPT-5.4 (Self-Consistency) & 28.2 & 40.8 & 29.6 & 0.6 & 30.0 & 32.0 & 28.2 & 13 \\
ToT (GPT-5.4) & 31.0 & 33.2 & 30.2 & 10.6 & 26.0 & 32.7 & 28.6 & 11 \\
\hline
\rowcolor{myblue}
\textit{\textbf{Open-source General Models}} & & & & & & & & \\
MiMo-VL-7B-SFT~\cite{li2025xiaomi} & 21.8 & 18.8 & 31.6 & 0.0 & 36.0 & 39.3 & 28.3 & 12 \\
MiMo-VL-7B-RL~\cite{li2025xiaomi} & 26.0 & 14.4 & 22.0 & 0.0 & 26.0 & 36.0 & 24.6 & 22 \\
LLaMA-3.2V-11B~\cite{grattafiori2024llama} & 23.0 & 10.8 & 16.8 & 21.0 & 26.0 & 28.0 & 22.7 & 26 \\
LLaMA-3.2V-11B-CoT~\cite{xu2025llava} & 21.0 & 14.4 & 19.6 & 17.6 & 22.0 & 33.3 & 24.3 & 25 \\
Kimi-VL-16B-A3B~\cite{team2025kimi} & 24.2 & 9.6 & 22.4 & 24.6 & 24.0 & 30.7 & 24.6 & 22 \\
Kimi-VL-16B-A3B-Thinking~\cite{team2025kimi} & 17.8 & 30.0 & 16.0 & 9.8 & 16.0 & 39.3 & 26.0 & 21 \\
Deepseek-VL2-27B-A4.5~\cite{wu2024deepseek} & 17.4 & 9.4 & 27.2 & 25.8 & 28.0 & 34.7 & 26.5 & 19 \\
Qwen3-VL-8B~\cite{bai2025qwen3} & 13.0 & 15.0 & 5.4 & 0.0 & 24.0 & 36.7 & 20.9 & 28 \\
Qwen3-VL-32B~\cite{bai2025qwen3} & 11.4 & 38.6 & 28.8 & 17.8 & 26.0 & 34.0 & 28.1 & 14 \\
LLaVA-NeXT-Video-34B~\cite{zhang2410video} & 17.0 & 12.2 & 21.4 & 5.0 & 18.0 & 34.7 & 22.2 & 27 \\
InternVL3-78B~\cite{zhu2025internvl3} & 21.4 & 22.8 & 32.2 & 42.2 & 28.0 & 34.7 & 31.3 & 6 \\
\hline
\rowcolor{myblue}
\textit{\textbf{Open-source SI Models}} & & & & & & & & \\
SpaceLLaVA-13B~\cite{chen2024spatialvlm} & 26.2 & 23.0 & 16.0 & 6.6 & 26.0 & 41.3 & 27.7 & 15 \\
SpaceR-7B~\cite{ouyang2025spacer} & 13.0 & 25.0 & 33.2 & 16.6 & 22.0 & 35.3 & 27.0 & 18 \\
SenseNova-SI-1.1-Qwen2.5-VL-7B~\cite{cai2026scaling} & 15.6 & 21.6 & 23.8 & 16.0 & 22.0 & 40.0 & 27.4 & 16 \\
SenseNova-SI-1.1-Qwen3-VL-8B~\cite{cai2026scaling} & 4.0 & 20.0 & 18.8 & 11.4 & 34.0 & 36.0 & 24.5 & 24 \\
SenseNova-SI-1.3-InternVL3-8B~\cite{cai2026scaling} & 32.8 & 20.0 & 23.2 & 32.8 & 24.0 & 39.3 & 31.3 & 6 \\
VST-7B-SFT~\cite{yang2025visual} & 19.4 & 27.2 & 24.2 & 2.4 & 26.0 & 36.7 & 26.2 & 20 \\
Spatial-MLLM-v1.1-135K~\cite{wu2026spatial} & 27.2 & 34.6 & 16.6 & 27.0 & 20.0 & 34.7 & 28.7 & 9 \\
Spatial-MLLM-v1.1-820K~\cite{wu2026spatial} & 42.8 & 30.0 & 23.8 & \textbf{42.4} & 18.0 & 34.0 & 32.4 & 4 \\
\hline
GPT-5.4 (Stepwise) & 30.6 & 33.4 & 31.0 & 4.2 & 36.0 & 31.3 & 28.7 & 9 \\
\rowcolor{mygreen}
\textbf{+ Advantage-Guided Gate} & \textbf{45.8}\upgain{15.2} & 38.4\upgain{5.0} & 38.4\upgain{7.4} & 39.4\upgain{35.2} & \textbf{44.0}\upgain{8.0} & 48.0\upgain{16.7} & 43.8\upgain{15.1} & 2 \\
Gemini-3.1-pro-preview (Stepwise) & 42.6 & 43.2 & 29.2 & 15.2 & 40.0 & 30.7 & 34.0 & 3 \\
\rowcolor{mygreen}
\textbf{+ Advantage-Guided Gate} & 42.6\zerogain{0.0} & \textbf{44.0}\upgain{0.8} & \textbf{43.6}\upgain{14.4} & 32.6\upgain{17.4} & 40.0\zerogain{0.0} & \textbf{53.3}\upgain{22.6} & \textbf{45.4}\upgain{11.4} & 1 \\
\bottomrule

\end{tabular}
}
\label{tab:ScanNetPP}
\vspace{-0.5cm}
\end{table*}

Compared with reasoning-time augmentation methods such as CoT~\cite{wei2022chain}, Self-Consistency~\cite{wangself}, and Tree-of-Thought (ToT)~\cite{yao2023tree}, the key distinction of the method proposed in this paper lies in its explicit modeling of the selection mechanism for candidate reasoning steps, aligning its optimization objective with the final task return. CoT relies on the autoregressive expansion of a single reasoning chain and is susceptible to the propagation of early errors; Self-Consistency improves prediction stability through multiple sampling and answer aggregation, but does not directly evaluate or intervene in intermediate reasoning steps; although ToT introduces branch expansion and node filtering, its selection process typically relies on model self-evaluation, prompt scoring, or predefined search rules, and thus retains a strong heuristic nature—local scores may not accurately reflect the actual contribution of the current step to the quality of the final answer. Consequently, expanding the candidate search space may still retain branches that appear reasonable on the surface but fail to lead to the correct answer. In contrast, we construct step-level and trajectory-level supervision based on the quality of subsequently reachable answers in the reasoning tree, and estimate the values of prefix nodes and terminal answer nodes via SAG and TAG, respectively, thereby transforming candidate screening into a supervised learning problem focused on the quality of the final answer. The performance advantages of ToT over other methods demonstrate that the key to spatial reasoning search lies not only in the generation and expansion of candidate paths, but also in the ability to effectively select candidates based on clear objectives aligned with task returns.

\begin{table*}[!t]
\centering
\caption{Quantitative comparison on Matterport3D. Numerical and multiple-choice questions are evaluated by MRA and accuracy, respectively. Avg. is used for ranking. \textbf{Bold} denotes the best result in each column. Green rows show results with Advantage-Guided Gate, with colored values indicating absolute changes from the corresponding Stepwise baseline. SI denotes spatial instruction tuning. All scores are percentages.}
\resizebox{15.5cm}{!}{
\begin{tabular}{l|cccc|cc|cc}
\toprule
\multirow{2}{*}{Methods}
& \multicolumn{4}{c|}{Numerical}
& \multicolumn{2}{c|}{Multiple-Choice}
& \multirow{2}{*}{Avg.}
& \multirow{2}{*}{Rank} \\
\cline{2-7}
& Obj. Cnt. & Abs. Dist. & Obj. Size & Room Size
& Rel. Dist. & Rel. Dir.
& & \\
\hline
\rowcolor{myblue}
\textit{\textbf{Proprietary Models (API)}} & & & & & & & & \\
Gemini-2.5-pro & 23.8 & 30.0 & 30.8 & 1.4 & 34.0 & 35.3 & 28.3 & 9 \\
Gemini-3.1-pro-preview (CoT) & 33.2 & 25.0 & 36.8 & 2.8 & 36.0 & 40.7 & 32.0 & 6 \\
GPT-5.4 (CoT) & 11.2 & 29.0 & 30.8 & 1.6 & 22.0 & 41.3 & 27.3 & 15 \\
GPT-5.4 (Self-Consistency) & 14.2 & 27.4 & 29.4 & 0.0 & 26.0 & 38.0 & 26.4 & 17 \\
ToT (GPT-5.4) & 9.0 & 35.2 & 32.2 & 9.8 & 38.0 & 34.0 & 28.3 & 9 \\

\hline
\rowcolor{myblue}
\textit{\textbf{Open-source General Models}} & & & & & & & & \\
MiMo-VL-7B-SFT~\cite{li2025xiaomi} & 27.2 & 23.8 & 27.6 & 0.0 & 34.0 & 37.3 & 28.1 & 11 \\
MiMo-VL-7B-RL~\cite{li2025xiaomi} & 30.0 & 15.4 & 17.0 & 0.0 & 30.0 & 34.0 & 24.3 & 21 \\
LLaMA-3.2V-11B~\cite{grattafiori2024llama} & 21.0 & 12.6 & 28.6 & 6.2 & 12.0 & 43.0 & 26.6 & 16 \\
LLaMA-3.2V-11B-CoT~\cite{xu2025llava} & 13.0 & 7.8 & 17.8 & 14.4 & 24.0 & 31.3 & 21.4 & 27 \\
Kimi-VL-16B-A3B~\cite{team2025kimi} & 12.2 & 9.6 & 20.6 & 21.6 & 24.0 & 29.3 & 22.0 & 26 \\
Kimi-VL-16B-A3B-Thinking~\cite{team2025kimi} & 23.4 & 23.0 & 30.0 & 10.2 & 32.0 & 34.0 & 27.6 & 13 \\
Deepseek-VL2-27B-A4.5~\cite{wu2024deepseek} & 20.4 & 4.8 & 26.4 & 20.0 & 24.0 & 41.3 & 27.5 & 14 \\
Qwen3-VL-8B~\cite{bai2025qwen3} & 12.4 & 10.6 & 10.4 & 0.0 & 14.0 & 36.7 & 19.7 & 28 \\
Qwen3-VL-32B~\cite{bai2025qwen3} & 9.8 & 31.6 & 27.0 & 3.6 & 18.0 & 30.7 & 22.8 & 24 \\
LLaVA-NeXT-Video-34B~\cite{zhang2410video} & 12.6 & 5.8 & 24.0 & 2.2 & 24.0 & 36.7 & 22.3 & 25 \\
InternVL3-78B~\cite{zhu2025internvl3} & 7.8 & 26.6 & 24.2 & 22.4 & 26.0 & 32.0 & 25.4 & 20 \\

\hline
\rowcolor{myblue}
\textit{\textbf{Open-source SI Models}} & & & & & & & & \\
SpaceLLaVA-13B~\cite{chen2024spatialvlm} & 23.2 & 13.0 & 19.8 & 4.8 & 28.0 & 32.7 & 23.3 & 23 \\
SpaceR-7B~\cite{ouyang2025spacer} & 6.0 & 11.2 & 34.0 & \textbf{23.4} & 14.0 & 35.3 & 24.3 & 21 \\
SenseNova-SI-1.1-Qwen2.5-VL-7B~\cite{cai2026scaling} & 13.8 & 18.2 & 16.6 & 17.0 & 24.0 & 38.7 & 25.7 & 19 \\
SenseNova-SI-1.1-Qwen3-VL-8B~\cite{cai2026scaling} & 9.8 & 27.8 & 28.4 & 0.8 & 24.0 & 38.7 & 25.9 & 18 \\
SenseNova-SI-1.3-InternVL3-8B~\cite{cai2026scaling} & 21.0 & 29.8 & 33.0 & 22.8 & 14.0 & 44.7 & 31.8 & 7 \\
VST-7B-SFT~\cite{yang2025visual} & 26.0 & 14.0 & 33.4 & 4.8 & 24.0 & 42.0 & 28.5 & 8 \\
Spatial-MLLM-v1.1-135K~\cite{wu2026spatial} & 28.4 & 37.2 & 41.6 & 13.8 & 12.0 & 44.7 & 33.4 & 4 \\
Spatial-MLLM-v1.1-820K~\cite{wu2026spatial} & 23.8 & 36.6 & \textbf{42.8} & 14.8 & 24.0 & 38.7 & 32.3 & 5 \\

\hline
GPT-5.4 (Stepwise) & 24.2 & 39.6 & 31.4 & 10.2 & 16.0 & 33.3 & 27.7 & 12 \\
\rowcolor{mygreen}
\textbf{+ Advantage-Guided Gate}
& 14.8\downdrop{9.4}
& 41.6\upgain{2.0}
& 38.6\upgain{7.2}
& 14.6\upgain{4.4}
& 36.0\upgain{20.0}
& 43.3\upgain{10.0}
& 34.5\upgain{6.8}
& 3 \\
Gemini-3.1-pro-preview (Stepwise) & \textbf{33.2} & 31.2 & 31.4 & 10.8 & 48.0 & 41.3 & 34.8 & 2 \\
\rowcolor{mygreen}
\textbf{+ Advantage-Guided Gate}
& 30.6\downdrop{2.6}
& \textbf{42.6}\upgain{11.4}
& 38.0\upgain{6.4}
& 11.8\upgain{1.0}
& \textbf{54.0}\upgain{6.0}
& \textbf{50.7}\upgain{9.4}
& \textbf{41.9}\upgain{7.1}
& 1 \\

\bottomrule
\end{tabular}
}
\label{tab:Matterport3D}
\vspace{-0.3cm}
\end{table*}

The method described in this paper adopts a different capability enhancement paradigm from spatial instruction fine-tuning models such as SpaceLLaVA~\cite{chen2024spatialvlm}, SpaceR~\cite{ouyang2025spacer}, SenseNova-SI~\cite{cai2026scaling}, VST~\cite{yang2025visual}, and Spatial-MLLM~\cite{wu2026spatial}. The aforementioned methods primarily update model parameters using spatial instruction data, encoding spatial knowledge and task patterns into the model itself; the coverage of the training data often influences their performance, the distribution of tasks, and the architecture of the base model. In contrast, this paper does not update the original parameters of the evaluated MLLM. Instead, it uses a gated model as a plug-and-play inference control module to perform value evaluation and dynamic selection within the candidate spaces generated by the base model. Specifically, the gated versions of GPT-5.4 and Gemini-3.1-pro-preview outperformed the best-performing spatial instruction fine-tuned models across all four benchmarks, demonstrating that general-purpose MLLMs can achieve competitive spatial reasoning performance simply by reshaping the selection distribution of candidate reasoning paths. This consistent advantage further illustrates that the generative space of existing MLLMs already contains a certain proportion of high-quality spatial reasoning trajectories; 
however, without a reliable mechanism for trajectory identification and selection, these potentially correct trajectories are difficult to retain and convert into consistent final predictions.

\begin{table*}[!t]
\centering
\caption{Quantitative comparison on HM3D. Numerical and multiple-choice questions are evaluated by MRA and accuracy, respectively. Avg. is used for ranking. \textbf{Bold} denotes the best result in each column. Green rows show results with Advantage-Guided Gate, with colored values indicating absolute changes from the corresponding Stepwise baseline. SI denotes spatial instruction tuning. All scores are percentages.}
\resizebox{15.5cm}{!}{
\begin{tabular}{l|cccc|cc|cc}
\toprule
\multirow{2}{*}{Methods}
& \multicolumn{4}{c|}{Numerical}
& \multicolumn{2}{c|}{Multiple-Choice}
& \multirow{2}{*}{Avg.}
& \multirow{2}{*}{Rank} \\
\cline{2-7}
& Obj. Cnt. & Abs. Dist. & Obj. Size & Room Size
& Rel. Dist. & Rel. Dir.
& & \\
\hline
\rowcolor{myblue}
\textit{\textbf{Proprietary Models (API)}} & & & & & & & & \\
Gemini-2.5-pro & 15.8 & 17.0 & 33.8 & 10.2 & 30.0 & 35.3 & 26.6 & 10 \\
Gemini-3.1-pro-preview (CoT) & 19.0 & 20.8 & 30.4 & 2.6 & 38.0 & 40.0 & 28.9 & 6 \\
GPT-5.4 (CoT) & 14.0 & 18.0 & 20.0 & 11.8 & 26.0 & 34.7 & 24.2 & 20 \\
GPT-5.4 (Self-Consistency) & 11.0 & 18.6 & 20.2 & 0.6 & 22.0 & 38.7 & 23.5 & 22 \\
ToT (GPT-5.4) & 13.8 & 16.2 & 27.2 & 10.0 & 40.0 & 44.0 & 29.9 & 3 \\

\hline
\rowcolor{myblue}
\textit{\textbf{Open-source General Models}} & & & & & & & & \\
MiMo-VL-7B-SFT~\cite{li2025xiaomi} & 15.4 & 4.0 & 33.8 & 0.0 & 28.0 & 40.7 & 25.4 & 14 \\
MiMo-VL-7B-RL~\cite{li2025xiaomi} & 17.8 & 9.4 & 25.4 & 0.0 & 14.0 & 31.3 & 20.1 & 25 \\
LLaMA-3.2V-11B~\cite{grattafiori2024llama} & 13.0 & 7.4 & 20.2 & 11.0 & 30.0 & 35.3 & 23.5 & 22 \\
LLaMA-3.2V-11B-CoT~\cite{xu2025llava} & 8.4 & 4.2 & 17.2 & 16.6 & 22.0 & 38.7 & 25.5 & 13 \\
Kimi-VL-16B-A3B~\cite{team2025kimi} & 13.2 & 4.4 & 17.8 & 20.8 & 20.0 & 27.3 & 19.8 & 26 \\
Kimi-VL-16B-A3B-Thinking~\cite{team2025kimi} & 10.2 & 7.0 & 31.8 & 15.2 & 22.0 & 38.0 & 25.0 & 18 \\
Deepseek-VL2-27B-A4.5~\cite{wu2024deepseek} & 22.0 & 7.4 & 24.6 & 26.8 & 22.0 & 33.3 & 25.4 & 14 \\
Qwen3-VL-8B~\cite{bai2025qwen3} & 7.2 & 10.6 & 8.0 & 0.0 & 22.0 & 34.7 & 19.0 & 28 \\
Qwen3-VL-32B~\cite{bai2025qwen3} & 7.4 & 20.2 & 18.6 & 10.8 & 16.0 & 35.3 & 22.4 & 24 \\
LLaVA-NeXT-Video-34B~\cite{zhang2410video} & 9.0 & 5.2 & 21.0 & 0.0 & 16.0 & 35.3 & 19.7 & 27 \\
InternVL3-78B~\cite{zhu2025internvl3} & 5.4 & 15.2 & 26.6 & \textbf{27.6} & 40.0 & 32.7 & 26.6 & 10 \\

\hline
\rowcolor{myblue}
\textit{\textbf{Open-source SI Models}} & & & & & & & & \\
SpaceLLaVA-13B~\cite{chen2024spatialvlm} & 21.6 & 12.6 & 15.6 & 5.2 & 36.0 & 38.7 & 25.9 & 12 \\
SpaceR-7B~\cite{ouyang2025spacer} & 7.6 & 13.0 & 31.8 & 12.4 & 24.0 & 37.3 & 25.1 & 17 \\
SenseNova-SI-1.1-Qwen2.5-VL-7B~\cite{cai2026scaling} & 19.0 & 16.0 & 23.2 & 15.0 & 30.0 & 42.7 & 28.9 & 6 \\
SenseNova-SI-1.1-Qwen3-VL-8B~\cite{cai2026scaling} & 5.6 & 11.0 & 20.0 & 12.8 & 24.0 & 41.3 & 24.7 & 19 \\
SenseNova-SI-1.3-InternVL3-8B~\cite{cai2026scaling} & 23.8 & 11.8 & 32.6 & 12.4 & 32.0 & 40.0 & 29.1 & 5 \\
VST-7B-SFT~\cite{yang2025visual} & 21.8 & 15.4 & 28.2 & 5.0 & 10.0 & 36.0 & 23.6 & 21 \\
Spatial-MLLM-v1.1-135K~\cite{wu2026spatial} & 18.0 & 15.0 & 29.0 & 13.8 & 20.0 & 39.3 & 26.7 & 9 \\
Spatial-MLLM-v1.1-820K~\cite{wu2026spatial} & 23.8 & 20.4 & 27.0 & 4.4 & 26.0 & 44.0 & 29.2 & 4 \\

\hline
GPT-5.4 (Stepwise) & 11.2 & 18.2 & 24.0 & 12.4 & 34.0 & 34.0 & 25.2 & 16 \\
\rowcolor{mygreen}
\textbf{+ Advantage-Guided Gate}
& 18.0\upgain{6.8}
& 22.8\upgain{4.6}
& \textbf{35.2}\upgain{11.2}
& 22.8\upgain{10.4}
& \textbf{48.0}\upgain{14.0}
& 46.3\upgain{12.3}
& 36.1\upgain{10.9}
& 2 \\
Gemini-3.1-pro-preview (Stepwise) & 25.6 & 21.4 & 25.8 & 3.2 & 38.0 & 33.3 & 26.8 & 8 \\
\rowcolor{mygreen}
\textbf{+ Advantage-Guided Gate}
& \textbf{29.8}\upgain{4.2}
& \textbf{23.0}\upgain{1.6}
& 34.0\upgain{8.2}
& 15.4\upgain{12.2}
& 40.0\upgain{2.0}
& \textbf{53.0}\upgain{19.7}
& \textbf{38.0}\upgain{11.2}
& 1 \\

\bottomrule
\end{tabular}
}
\label{tab:HM3D}
\end{table*}

Task-level results further reveal the differentiated benefits of advantage-guided gating across different spatial reasoning tasks. Across four datasets and two baseline models, the average gain from gating reached 15.7 and 15.8 percentage points for room size estimation and relative direction judgment, respectively, and also achieved improvements of 9.9 and 7.7 percentage points for relative distance judgment and object size estimation; in contrast, the average gains for object counting and absolute distance estimation were only 3.4 and 3.1 percentage points, respectively. This discrepancy is closely related to the primary sources of error for different tasks: the former typically rely on the aggregation of cross-view information and the inference of continuous scale or spatial relationships, and are therefore more susceptible to the accumulation of intermediate errors; SAG’s filtering of low-value prefixes and TAG’s final evaluation of complete trajectories can effectively mitigate such error propagation; In contrast, object counting and absolute distance estimation are more heavily influenced by underlying perception errors such as object omissions, cross-view instance association, and visual scale calibration; therefore, the benefits achievable solely through reasoning-based path selection are relatively limited. The above results indicate that the method proposed in this paper primarily enhances candidate selection and error control capabilities during multi-step reasoning, rather than replacing underlying visual perception and spatial measurement.

Cross-dataset results further validate the method’s adaptability to different scene representations and visual acquisition methods. ScanNet and ScanNet++ use video keyframes from real-world scan sequences as visual input. In contrast, Matterport3D and HM3D use region-aligned NFOV multi-view images and room-level multi-view images, respectively, obtained through coverage-aware sampling. These datasets exhibit significant differences in viewpoint continuity, image overlap, scene complexity, and spatial distribution. Nevertheless, answer-guided gating consistently achieves overall gains across all four benchmarks, indicating that it learns not fixed heuristic rules dependent on a specific dataset or visual sampling method, but rather a relatively stable correlation between reasoning steps and the quality of the final answer. In other words, as long as the base model can generate candidate inference paths with a certain degree of diversity, the proposed method can use value estimation to increase the probability that high-quality paths are retained and ultimately selected.

In summary, the results of the main experiment support the core hypothesis of this paper: the performance bottleneck in open-space reasoning lies not only in generating potentially correct reasoning paths, but also in continuously identifying and retaining high-value branches from a large-scale candidate space. By integrating the screening of intermediate steps with the selection of the final answer into a value-learning framework driven by terminal returns, the method proposed in this paper suppresses the expansion of low-quality reasoning branches, reshapes the effective output distribution of the foundation model, and increases the relative proportion of high-quality reasoning trajectories in the final prediction.

\begin{table}[t]
\centering
\caption{Ablation study of the proposed framework. The upper part
evaluates the contributions of SAG and TAG, while the lower part
analyzes the shared initialization and TAG specialization designs.
Unless otherwise specified, TAG adopts the default multi-head
architecture.}
\resizebox{8.5cm}{!}{
\label{tab:ablation}
\setlength{\tabcolsep}{10pt}
\begin{tabular}{lcccc}
\toprule
Variant & SAG & TAG & Avg. & $\Delta$ \\
\midrule

\multicolumn{5}{l}{\textit{Gating component analysis}} \\

Stepwise baseline & \ding{55} & \ding{55} & 34.1 & -- \\
SAG only & \ding{51} & \ding{55} & 35.8 & +1.7 \\
TAG only & \ding{55} & \ding{51} & 37.5 & +3.4 \\
Full model & \ding{51} & \ding{51} & \textbf{44.6} & +10.5 \\

\midrule

\multicolumn{5}{l}{\textit{Initialization and head design}} \\

w/o SWA shared initialization & \ding{51} & \ding{51} & 43.3 & -1.3 \\
TAG with a single head & \ding{51} & \ding{51} & 40.1 & -4.5 \\
Full model & \ding{51} & \ding{51} & \textbf{44.6} & -- \\

\bottomrule
\end{tabular}}
\vspace{-0.5cm}
\end{table}

\subsection{Ablation on Gating Components and Task Specialization}
Table ~\ref{tab:ablation} analyzes the contributions of various gating components and training designs to overall performance. When SAG and TAG were introduced separately, the overall scores improved by 1.7 and 3.4 percentage points, respectively, compared to the Stepwise baseline, indicating that both have independent effects but differ fundamentally in their optimization objectives. SAG focuses on the reasoning expansion process; its primary goal is not to directly select the final answer, but to increase the density of high-quality reasoning trajectories in the candidate solution space by continuously filtering out low-value prefixes. As shown in the theoretical analysis in Appendix~\ref{appendix:open-ended_reasoning_space_reshaping}, SAG can contract the original open-ended reasoning space into a higher-quality feasible subspace. However, an increase in the number of high-quality answers in the candidate set does not necessarily translate directly into final accuracy; in the absence of a reliable mechanism for selecting the final answer, even if the proportion of high-quality answers increases significantly, the final prediction may still fail to select the correct answer. 
Therefore, SAG alone yields only modest direct gains.

TAG, on the other hand, directly evaluates the value of the complete inference trajectory and its final answer, thereby achieving a more pronounced performance improvement within the original candidate space. However, due to the limited number of high-quality trajectories in the original candidate space, TAG’s selection ceiling remains constrained by the initial quality of the candidates. When SAG and TAG are enabled simultaneously, the overall score reaches 44.6, representing a 10.5 improvement over the baseline and significantly exceeding the gains achieved when either module is used alone. This result reveals the complementary relationship between the two: SAG first increases the density of high-quality trajectories in the candidate solution space, and TAG subsequently identifies and selects the optimal answer from this reshaped candidate space, thereby fully translating the improvement in solution space quality into final prediction performance. In other words, SAG addresses the issue of “whether enough high-quality candidates can be generated,” while TAG addresses the issue of “whether high-quality answers can be reliably selected from the candidates”; together, they constitute a complete value-guided process ranging from space reshaping to final decision-making.

The lower half of the table further validates the effectiveness of shared initialization and task-specific design. Removing SWA’s shared initialization resulted in a 1.3 percentage-point decline in overall performance, indicating that although the six types of spatial question-answering tasks have different answer formats and reasoning objectives, they share fundamental spatial knowledge across object recognition, cross-view information aggregation, spatial relationship understanding, and scale modeling. By aggregating the shared parameters from the optimal states of each task, SWA preserves common spatial representations across tasks and provides a more robust general initialization. In contrast, replacing the multi-head TAG with a single-head structure results in a 4.5 percentage-point drop in performance, indicating that different problem types exhibit distinct task-specific preferences for evaluating termination trajectories. A single-head TAG must evaluate candidate trajectories with different answer formats, error patterns, and reasoning bases using a unified standard, which can easily lead to cross-task interference; in contrast, a multi-head TAG learns task-specific end-state value functions for each problem type based on shared spatial representations, achieving further task-oriented optimization. Therefore, SWA extracts transferable general spatial knowledge from six categories of spatial question-answering tasks. At the same time, multi-head TAG further learns task-specific value judgment criteria on this basis, forming a hierarchical modeling mechanism that progresses from shared spatial representations to specialized end-state evaluation.

\begin{figure*}[t]
    \centering
    \includegraphics[width=0.8\linewidth]{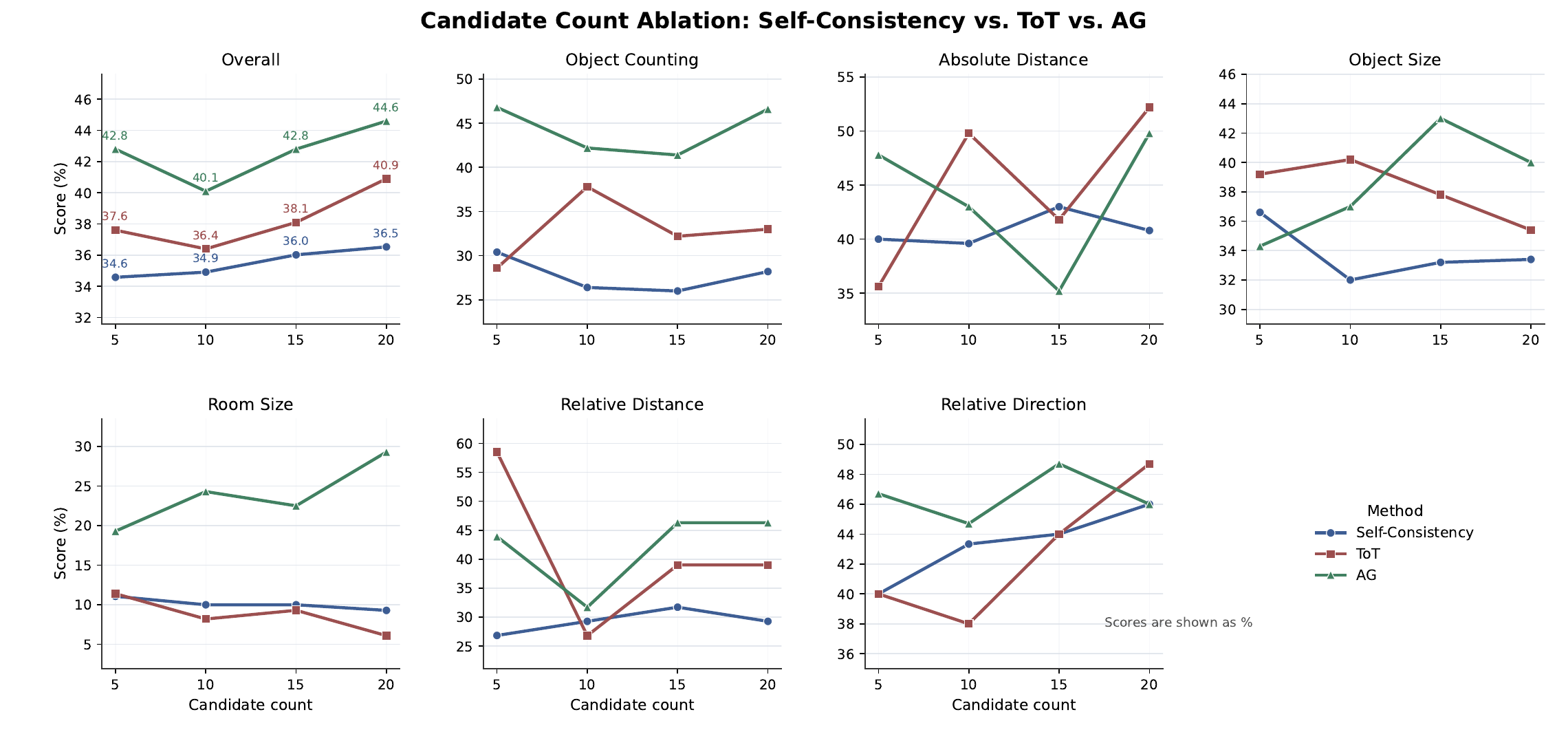}
    \caption{Effect of candidate budget on Self-Consistency, Tree-of-Thought (ToT), and the proposed Advantage-Guided Gate (AG) on ScanNet. Overall and task-specific performance is reported using 5, 10, 15, and 20 candidates. The horizontal axis denotes the number of candidates, and the vertical axis denotes the score (\%). AG achieves the best overall performance across all budgets, indicating more effective use of additional candidates.}
    \label{fig: candidate count ablation}
    \vspace{-0.5cm}
\end{figure*}

\subsection{Effect of Candidate Budget}
Figure. ~\ref{fig: candidate count ablation} compares the overall performance of Self-Consistency, ToT, and Advantage-Guided Gate (AG) under different numbers of candidates. All three methods achieve inference gains by expanding the candidate search space; therefore, the number of candidates directly determines the search budget during inference. The results show that expanding the search space only consistently translates into performance improvements when combined with an effective candidate utilization mechanism. As the number of candidates increased from 5 to 20, Self-Consistency’s overall score rose only from 34.6 to 36.5, indicating that while repeated sampling expanded the coverage of possible answers, simple aggregation struggled to utilize the additional candidates fully; ToT’s overall score fluctuated between 36.4 and 40.9, indicating that branch-and-bound search based on heuristic node evaluation is sensitive to the number of candidates, and expanding the search space does not necessarily yield stable gains. In contrast, AG achieved higher performance across all candidate budgets, reaching 42.8 with only 5 candidates—already surpassing the results of Self-Consistency and ToT with 20 candidates. When the number of candidates increased to 20, AG’s overall score further improved to 44.6, exceeding SC@20 and ToT@20 by 8.1 and 3.7 percentage points, respectively. This demonstrates that AG not only expands the coverage of potential high-quality inference paths by increasing the number of candidates but also effectively identifies and retains candidates where they are most likely to lead to the correct answer through a value-guided mechanism, thereby more fully converting the search budget into final performance. Taking both the search budget and the best overall performance within the current candidate set into account, this paper sets the default number of candidates to 20.

\begin{figure*}[t]
    \centering
    \includegraphics[width=0.7\linewidth]{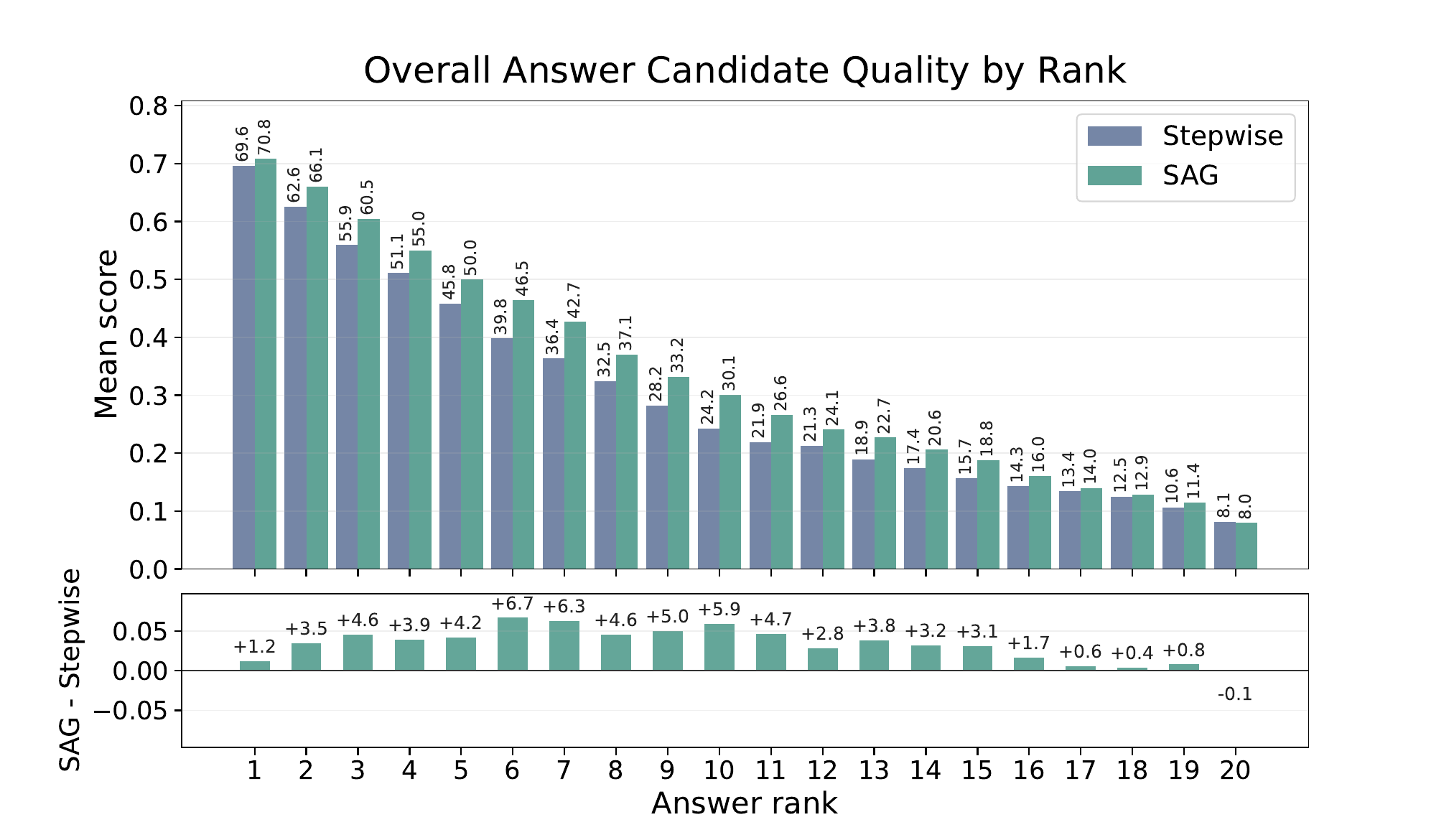}
    \caption{Comparison of candidate solution-space quality between Stepwise reasoning and Stepwise reasoning with SAG on ScanNet. For each question, 20 complete reasoning trajectories and their answers are scored against the ground truth and ranked in descending order, where Rank~1 denotes the highest-quality candidate. Scores at the same rank are then averaged over the test set. The upper panel reports the mean candidate quality at each rank, while the lower panel shows the absolute difference between SAG and Stepwise. SAG improves candidate quality across nearly all ranks, indicating a systematic reshaping of the candidate solution space rather than an isolated improvement of the best candidate. All scores are reported as percentages.}
    \label{fig: Stepwise_vs_SAG_distribution_visualization}
    \vspace{-0.5cm}
\end{figure*}

\subsection{Candidate Solution Space Reshaping by SAG}
Figure ~\ref{fig: Stepwise_vs_SAG_distribution_visualization} compares the quality of the candidate solution space between Stepwise reasoning and the approach incorporating SAG, to examine the impact of step-level gating on the final candidate distribution in isolation. Specifically, under the same reasoning budget and decoding configuration, both methods generate 20 complete reasoning trajectories and their answers for each problem, and score each candidate based on the ground truth. Next, for each question, the 20 candidates are sorted by their true scores from highest to lowest, where Rank 1 represents the highest-quality solution in the current set of candidates and Rank 20 represents the lowest-quality solution; finally, the scores of candidates at the same rank position in the test set are averaged. The upper half of the figure shows the average candidate quality at each rank position for the two methods, while the lower half shows the rank-wise gain of SAG relative to Stepwise. Since the ranking is based solely on the direct scores of the answers—without involving TAG prediction or selection—this experiment directly reflects SAG’s role in reshaping the candidate solution space itself.

The experimental results show that SAG achieved performance improvements across 19 out of 20 ranking positions, indicating that its improvements are not concentrated on a small number of candidates but rather have a broad impact across the entire candidate distribution. Where candidates ranked 2nd through 15th all saw steady improvements, the gains were particularly significant for those ranked 6th through 10th, reaching 4.6–6.7 percentage points. Furthermore, the average quality of the top 5, top 10, and top 15 candidates improved by 3.5, 4.6, and 4.2 percentage points, respectively, and the average score of all 20 candidates rose from 30.0 to 33.4. This distribution pattern indicates that SAG’s effect is not limited to improving a single optimal candidate but also systematically enhances the medium- to high-quality regions of the candidate solution space. The top-ranked candidate represents the highest-quality solution that the base model can generate under the current search budget, and its performance is already close to the upper bound of the candidate space’s capability; on this basis, SAG still raises its average quality from 69.6 to 70.8, indicating that gating can further improve the upper limit of the candidate space’s quality. More notably, candidates ranked 2nd through 15th all saw consistent improvements, with the gains most pronounced in the mid-range. This indicates that SAG, by suppressing the continued expansion of low-value prefixes, directs more generated trajectories into the medium-to-high-quality range, thereby expanding the set of high-quality candidates and increasing their relative proportion within the entire solution space. This result is consistent with the theoretical analysis in the appendix~\ref{appendix:open-ended_reasoning_space_reshaping}: by progressively restricting low-value branches, SAG contracts and reshapes the open-ended inference space, increasing the relative density of high-quality inference trajectories and providing a more favorable solution space foundation for TAG to reliably identify high-quality answers from the candidate set in subsequent stages.

\subsection{Effect of the Top-\(\rho\) Labeling Ratio}
Figure ~\ref{fig: threshold ablation} analyzes the impact of the Top-$\rho$ labeling threshold on the gated training data. Specifically, samples within the Top-$\rho$ percentile of scores among candidates in the same layer are labeled as “accept,” while the remaining samples are labeled as “reject.” Overall performance exhibits a non-monotonic variation with respect to Top-$\rho$: when $\rho$ increases from 15\% to 25\%, the overall score rises from 43.7 to 44.6; when $\rho$ is further increased to 35\%, performance drops to 43.1. This indicates that the Top-$\rho$ threshold requires an appropriate trade-off between the selectivity of high-value samples and the coverage of positive supervision. A smaller $\rho$ value marks only a small number of high-scoring candidates as `accept,` which, while ensuring high-quality positive samples, may also misclassify candidates with reasonable inference potential as “reject,” causing the model to learn an overly conservative acceptance boundary. Conversely, a larger $\rho$ expands the coverage of positive samples but also includes more candidates of relatively limited value in the “accept” set, narrowing the quality gap between positive and negative samples and increasing supervisory noise. In contrast, $\rho=25\%$ strikes a more reasonable balance between candidate quality and supervision diversity, thereby achieving the best overall performance. Although different tasks vary in their sensitivity to $\rho$, considering both overall performance and cross-task stability, this paper adopts $\rho=25\%$ as the default setting.

\begin{figure}[t]
    \centering
    \includegraphics[width=\linewidth]{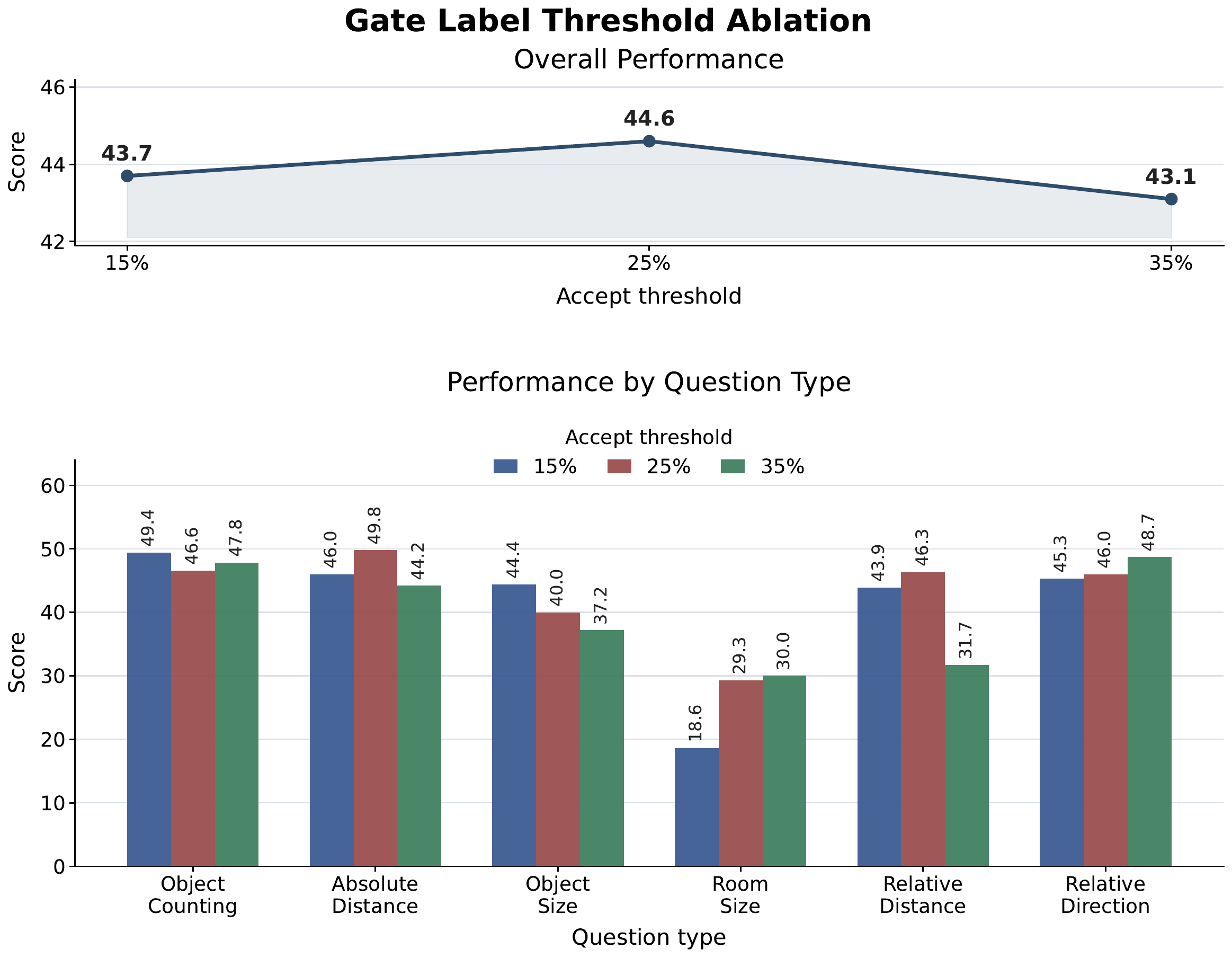}
    \caption{Effect of the Top-\(\rho\) labeling ratio used to construct SAG training labels on ScanNet. At each tree level, candidates ranked within the top \(\rho\) are labeled as accept, while the remaining candidates are labeled as reject. The upper panel reports overall performance under different labeling ratios, and the lower panel shows the corresponding task-specific results. The best performance is achieved at \(\rho=25\%\), indicating a better balance between positive-sample quality and supervision coverage. Scores are reported in \%.}
    \label{fig: threshold ablation}
    \vspace{-0.5cm}
\end{figure}

\section{Conclusion}

This work presented Advantage-Guided Gate (AG), a plug-and-play framework for improving open-ended spatial reasoning in MLLMs without modifying the evaluated model parameters. By constructing Monte Carlo reasoning trees, we transform terminal answer feedback into hierarchical supervision for intermediate prefixes and complete trajectories. The resulting Step-Advantage Gate filters low-value reasoning prefixes, while the Trajectory-Advantage Gate selects high-quality terminal trajectories. Extensive experiments on ScanNet, ScanNet++, Matterport3D, and HM3D demonstrate consistent improvements across different base models and visual observation settings. Theoretical analysis and ablation studies further show that SAG and TAG play complementary roles in reshaping the candidate solution space and improving final-answer selection. These results indicate that reliable spatial reasoning depends not only on generating diverse reasoning paths, but also on effectively identifying and retaining high-value trajectories. Future work will explore adaptive reasoning budgets and tighter integration between reasoning control and spatial perception.



\bibliographystyle{IEEEtran}
\bibliography{IEEEtrans}

\biogap

\begin{IEEEbiographynophoto}{Ling Lin} is currently pursuing the Ph.D. degree with the School of Artificial Intelligence and Data Science, University of Science and Technology of China (USTC). His research interests include multimodal large language models, out-of-distribution research, multimodal reasoning, and 3D scene understanding. He has published papers in several leading journals and conferences such as IEEE TIP, TNNLS, TCSVT, ICASSP, and ICME.
\end{IEEEbiographynophoto}

\biogap

\begin{IEEEbiographynophoto}{Yang Bai}
Yang Bai received the M.S. degree in computer science from Newcastle University, U.K., in 2018, and the Ph.D. degree in computer science from Durham University, U.K., in 2023. He is currently a Scientist at the Institute of High Performance Computing (IHPC), A*STAR, Singapore. His research interests include multimodal foundation models, vision--language learning, and medical artificial intelligence, with a focus on medical image understanding and clinical multimodal reasoning.
\end{IEEEbiographynophoto}

\biogap

\begin{IEEEbiographynophoto}{Congcong Zhu\,(Member, IEEE)} received the Ph.D. degree from the School of Computer Engineering and Science, Shanghai University, Shanghai, China, in 2022. He is currently an Associate Researcher at the Suzhou Institute for Advanced Research, University of Science and Technology of China, Suzhou, China. His research interests include embodied artificial intelligence, multimodal content generation, and multimodal understanding. He has published papers in several journals, including IEEE Transactions on Image Processing, IEEE Transactions on Multimedia, IEEE Transactions on Circuits and Systems for Video Technology, Pattern Recognition, and IEEE Transactions on Instrumentation and Measurement, as well as in leading conferences, including ICML, CVPR, ICCV, AAAI, and ACM MM. He was nominated for the Shanghai Computer Society’s Outstanding Doctoral Dissertation Award, received the Outstanding Graduate Student Award of Shanghai, and was recognized as an Outstanding Postdoctoral Fellow in Jiangsu Province, China.
\end{IEEEbiographynophoto}

\biogap

\begin{IEEEbiographynophoto}{Jiangming Shi}
is currently an Assistant Professor with the School of Computer Science and Technology, East China Normal University, China. He received the Ph.D. degree from Xiamen University, China, in 2026. His research interests include spatial intelligence foundation models, robust multimodal learning, and all-weather visual perception. He has published nearly ten first-authored papers in leading international journals and conferences, including NeurIPS, ICCV, ECCV, and AAAI. He was a recipient of the Best Paper Award at CCHI.
\end{IEEEbiographynophoto}

\biogap

\begin{IEEEbiographynophoto}{Meng Wang}
Meng Wang received the Ph.D. degree in signal and information processing from Soochow University, China, in 2021. He is currently an Assistant Professor at the Centre for Innovation \& Precision Eye Health and the Department of Ophthalmology, Yong Loo Lin School of Medicine, National University of Singapore. His research interests include artificial intelligence, medical image analysis, multimodal foundation models, and clinical AI, with a focus on developing trustworthy and clinically deployable AI methods for precision healthcare.
\end{IEEEbiographynophoto}

\biogap

\begin{IEEEbiographynophoto}{Yang Long\,(Senior Member, IEEE)}
Yang Long is an Associate Professor in the Department of Computer Science, Durham University, U.K. He is also an MRC Innovation Fellow aiming to design scalable AI solutions for large-scale healthcare applications. His research background is in the interdisciplinary field of computer vision and machine learning. He is passionate about unveiling the black-box of AI and transferring the knowledge to seek scalable, interactable, interpretable, and sustainable solutions for other disciplinary researches, e.g., physical activity, mental health, design, education, security, and geoengineering. He has authored/coauthored 30+ top-tier papers in refereed journals/conferences such as IEEE TPAMI, TIP, CVPR, AAAI, and ACM MM, and holds a patent and a Chinese National Grant.
\end{IEEEbiographynophoto}

\biogap

\begin{IEEEbiographynophoto}{Jingrun Chen} received the B.S. degree from Nanjing University, China, in 2005, and the PhD degree from the Chinese Academy of Sciences, China, in 2010. From 2010 to 2015, he served as a Visiting Assistant Professor at the University of California, Santa Barbara, and from 2015 to 2021, he was a Distinguished Professor at Soochow University. He is currently a Professor at the Suzhou Institute for Advanced Research/School of Mathematical Sciences at the University of Science and Technology of China. His main work has been published in academic journals on applied and computational mathematics such as the SIAM series, Mathematics of Computation, and Journal of Computational Physics, as well as in interdisciplinary academic journals like Nature Communications and Materials Horizons.    
\end{IEEEbiographynophoto}

\biogap

\begin{IEEEbiographynophoto}{Ling Shao\,(Fellow, IEEE)}
Ling Shao is a Distinguished Professor with the UCAS-Terminus AI Lab, University of Chinese Academy of Sciences, Beijing, China. He was the founder of the Inception Institute of Artificial Intelligence (IIAI) and the Mohamed bin Zayed University of Artificial Intelligence (MBZUAI), Abu Dhabi, UAE. His research interests include generative AI, vision and language, and AI for healthcare. He is a fellow of the IEEE, the IAPR, the BCS and the IET.
\end{IEEEbiographynophoto}

\biogap

\begin{IEEEbiographynophoto}{Huazhu Fu\,(Senior Member, IEEE)}
Huazhu Fu is a Principal Scientist at the Institute of High Performance Computing (IHPC), A*STAR, Singapore. He earned his Ph.D. from Tianjin University in 2013. His research focuses on medical image analysis, AI for healthcare, and trustworthy AI. He is an Associate Editor for several distinguished journals, including IEEE Transactions on Medical Imaging, IEEE Transactions on Neural Networks and Learning Systems, and IEEE Journal of Biomedical and Health Informatics.
\end{IEEEbiographynophoto}

\clearpage

\appendices

\setcounter{section}{0}
\setcounter{subsection}{0}
\setcounter{subsubsection}{0}

\renewcommand{\thesection}{\Alph{section}}
\renewcommand{\thesubsection}{\thesection.\arabic{subsection}}
\renewcommand{\thesubsubsection}{\thesubsection.\arabic{subsubsection}}

\makeatletter
\def\@seccntformat#1{\csname the#1\endcsname\quad}
\makeatother

\renewcommand{\theequation}{A.\arabic{equation}}
\setcounter{equation}{0}

\section{Policy Improvement Perspective on Step-wise Reasoning}
\label{Sec:theory}

\subsection{Step-wise Reasoning as a Finite-Horizon Decision Process}
We model the step-by-step reasoning process of multimodal large language models (MLLMs) as a finite-horizon decision process.

In the $t$-th inference step, the state is defined as:

\begin{equation}
    h_t=(x, q, \tau_{1:t-1}),
\end{equation}

\noindent where $x$ denotes the visual input, $q$ denotes the problem, and $\tau_{1:t-1}=(r_1,\ldots,r_{t-1})$ denotes the accepted reasoning prefix generated before step $t$.

The reasoning policy for the original MLLM is as follows:

\begin{equation}
    \pi_0 (r_t \mid h_t ).
\end{equation}

\noindent where $r_t$ denotes the reasoning step generated at the current state $h_t$, \textit{i.e.}, the current reasoning action.

The complete reasoning trajectory is recorded as follows:

\[
    \tau_{1:K}=(r_1,\dots,r_K).
\]

The corresponding terminal trajectory return is defined as:

\[
    R(\tau_{1:K}) \in [0,1].
\]

Furthermore, we define the true action-value function under the initial policy $\pi_0$:

\begin{equation}
    Q_t^{\pi_0}(h_t, r_t)=\mathbb{E}_{\tau\sim\pi_0}[R(\tau)\mid h_t, r_t],
\end{equation}

\noindent which represents the expected trajectory reward that can ultimately be obtained when, starting from the current state $h_t$, the inference step $r_t$ is selected, and subsequent inference continues according to the original policy $\pi_0$.

Accordingly, the state value function is defined as:

\begin{equation}
    V_t^{\pi_0}(h_t)=\mathbb{E}_{r_t\sim \pi_0}[Q_t^{\pi_0}(h_t, r_t)].
\end{equation}

Ideally, if the true action-value function $Q_t^{\pi_0}(h_t, r_t)$ is observable, an ideal gated policy can be defined as follows:

\begin{equation}
    \pi_T^*(r_t\mid h_t)=\frac{\pi_0(r_t\mid h_t)\mathds{1}[Q_t^{\pi_0} (h_t, r_t)\geq T]}{Z_T(h_t)},
\end{equation}

\noindent where the normalization term is defined as:

\begin{equation}
    Z_T(h_t)= \Pr_{r_t\sim\pi_0}[Q_t^{\pi_0} (h_t, r_t)\geq T],
\end{equation}

\noindent when $Z_T(h_t) > 0$, the aforementioned gating policy constitutes a valid probability distribution.

\subsection{Ideal Value-Gated Policy Improvement}
\label{app:proof_gating_monotonicity}

\begin{proposition}
\label{prop:gating_amplification}
\textbf{Ideal Gating Amplifies the Single-Step Expected Value.}

For any history state $h_t \in \mathcal{H}$ satisfying the condition $Z_T(h_t) > 0$, the optimal gating policy $\pi_T^*$ guarantees a non-negative enhancement over the baseline inference policy in terms of the single-step expected value, i.e.,

\begin{equation}
\mathbb{E}_{r_t \sim \pi_T^*(\cdot \mid h_t)} \left[ Q_t^{\pi_0}(h_t, r_t) \right] \ge \mathbb{E}_{r_t \sim \pi_0(\cdot \mid h_t)} \left[ Q_t^{\pi_0}(h_t, r_t) \right]. \label{eq:single_step_bound}
\end{equation}

By the definition of the state-value function, Eq.~\ref{eq:single_step_bound} can be equivalently formulated as:

\begin{equation}
\mathbb{E}_{r_t \sim \pi_T^*(\cdot \mid h_t)} \left[ Q_t^{\pi_0}(h_t, r_t) \right] \ge V_t^{\pi_0}(h_t).
\end{equation}
\end{proposition}

\begin{proof}
Let the random variable be

\[
    S=Q_t^{\pi_0}(h_t, r_t),\quad r_t\sim\pi_0.
\]

The ideal gate policy corresponds to the accept event:

\[
    A={S \geq T}.
\]

Therefore, the expected value after gating can be written as the conditional expected value:

\[
    \mathbb{E}[S\mid A].
\]

By the law of total expectation, we have:

\[
    \mathbb{E}[S] = \Pr (A) \mathbb{E}[S \mid A] + \Pr(A')\mathbb{E}[S \mid A'].
\]

Since all samples in event $A$ satisfy $S \geq T$, and all samples in event $A'$ satisfy $S < T$, we have:

\[
    \mathbb{E}[S \mid A] \geq \mathbb{E}[S \mid A'].
\]

This further yields:

\[
    \mathbb{E}[S \mid A] \geq \mathbb{E}[S].
\]

Substituting back into the original definition, we obtain:

\[
    \mathbb{E}_{r_t \sim \pi_T^*(\cdot \mid h_t)} \left[ Q_t^{\pi_0}(h_t, r_t) \right] \ge \mathbb{E}_{r_t \sim \pi_0(\cdot \mid h_t)} \left[ Q_t^{\pi_0}(h_t, r_t) \right].
\]

The above inequality holds strictly when there exists a strictly lower-value action in the set of rejected actions.
\end{proof}

\begin{proposition}
\label{prop:Gating_Enhances_Full-Trajectory}
\textbf{Ideal Gating Enhances Full-Trajectory Cumulative Return.}

For any state $h_t$, the gating policy satisfies:

\begin{equation}
    \mathbb{E}_{r_t\sim \pi_T^*(\cdot\mid h_t)}
\left[Q_t^{\pi_0}(h_t,r_t)\right]
\geq V_t^{\pi_0}(h_t).    
\end{equation}

\noindent then, the expected return of the complete trajectory of the gated reasoning policy will not be lower than that of the original reasoning policy, that is:

\begin{equation}
    J(\pi^*_T)=\mathbb{E}_{\tau\sim\pi^*_T}[R(\tau)] \geq \mathbb{E}_{\tau\sim\pi_0}[R(\tau)]=J(\pi_0).
\end{equation}
\end{proposition}

\begin{proof}
The finite-horizon residual value function is defined as:

\[
    V_t^\pi(h_t)=\mathbb{E}_{\tau\sim\pi}[R(\tau)\mid h_t].
\]

For any policy $\pi$, according to the Bellman decomposition, we have:

\begin{equation}
    V_t^\pi(h_t)=\mathbb{E}_{r_t\sim\pi}[R(h_t, r_t)+\mathbb{E}[V_{t+1}^\pi(h_{t+1})\mid h_t, r_t]].
\end{equation}

In the inference-generation scenario of this paper, intermediate reasoning steps do not yield immediate rewards; instead, the final reward is determined solely by the correctness of the answer upon termination of the complete reasoning trajectory. Consequently, for all non-terminal steps:

\[
    R(h_t, r_t) = 0, \quad t<K.
\]

Thus, the Bellman decomposition can be simplified to:

\begin{equation}
    V_t^\pi(h_t)=\mathbb{E}_{r_t\sim\pi}[\mathbb{E}[V_{t+1}^\pi(h_{t+1})\mid h_t, r_t]].
\end{equation}

We will now prove this using reverse induction.

At the termination step $K$, we have:

\[
    V_K^{\pi^*_T}(h_K)=\mathbb{E}_{r_K \sim\pi^*_T}[Q_K^{\pi_0}(h_K, r_K)].
\]

According to the gating hypothesis, the following holds:

\[
     V_K^{\pi^*_T}(h_K) \geq V_K^{\pi_0}(h_K).
\]

Suppose that for any step $t+1$, for any reachable state $h_{t+1}$, the following holds:

\[
    V_{t+1}^{\pi^*_T}(h_{t+1}) \geq V_{t+1}^{\pi_0}(h_{t+1}).
\]

Then:

\begin{equation}
    \begin{aligned}
    V_t^{\pi^*_T}(h_t) &= \mathbb{E}_{r_t \sim \pi^*_T} \left[ \mathbb{E} \left[ V_{t+1}^{\pi^*_T}(h_{t+1}) \;\middle|\; h_t, r_t \right] \right] \\
    &\ge \mathbb{E}_{r_t \sim \pi^*_T} \left[ \mathbb{E} \left[ V_{t+1}^{\pi_0}(h_{t+1}) \;\middle|\; h_t, r_t \right] \right] \\
    &= \mathbb{E}_{r_t \sim \pi^*_T} \left[ Q_t^{\pi_0}(h_t, r_t) \right] \\
    &\ge V_t^{\pi_0}(h_t).
    \end{aligned}
\end{equation}

Therefore:

\[
    V_t^{\pi^*_T}(h_t) \geq V_t^{\pi_0}(h_t).
\]

\noindent for all time steps.

Finally, from:

\[
    J(\pi) = V_1^\pi (h_1),
\]

\noindent we obtain:

\[
    J(\pi_T^*) \ge J(\pi_0).
\]

Therefore, ideal value gating can be viewed as a policy refinement process based on action-value estimates. It enhances the expected reward of the overall reasoning trajectory by filtering out low-value intermediate reasoning steps.
\end{proof}

\subsection{Reasoning Tree as Monte Carlo Policy Evaluation}
\label{Sec:Monte_Carlo_Policy_Evaluation}
Since the true action-value function $Q_t^{\pi_0}(h_t, r_t)$ cannot be directly observed, we approximate it using Monte Carlo rollouts on a reasoning tree.

Let the partial reasoning trajectory from step $1$ to step $t$ be denoted as:
\[
    \tau _{1:t} = (r_1, \dots, r_t),
\]

\noindent which corresponds to a state-action pair $(h_t, r_t)$.

The definition of true action value is:

\[
    Q_t^{\pi_0}(h_t, r_t) = \mathbb{E}_{\tau\sim\pi_0}[R(\tau)\mid h_t, r_t].
\]

Since the above expectations cannot be computed directly, we start from the current partial trajectory $\tau_{1:t}$ and sample multiple subsequent reasoning rollouts:

\[
    \tau^{(1)}, \tau^{(2)}, \dots, \tau^{(m)},
\]

\noindent where each rollout corresponds to a complete terminating reasoning trajectory and yields a termination return:

\[
    R(\tau^{(j)}).
\]

The Monte Carlo estimator is therefore defined as:

\begin{equation}
    \widehat{Q}_m(h_t, r_t) = \frac{1}{m}\sum_{j=1}^m R(\tau^{(j)}).
\end{equation}

By the law of large numbers, as $m \rightarrow \infty$, we have:

\[
    \widehat{Q}_m(h_t, r_t) \xrightarrow{a.s.} Q_t^{\pi_0}(h_t, r_t).
\]

Therefore, the supervisory signal constructed based on the reasoning tree rollout can essentially be viewed as a finite-sample Monte Carlo estimate of the latent action-value function.

Furthermore, by the Hoeffding inequality, for any $\epsilon > 0$, we have:

\begin{equation}
    \Pr (|\widehat{Q}_m(h_t, r_t)-Q_t^{\pi_0}(h_t, r_t)| \ge \epsilon) \leq 2 \exp (-2m\epsilon^2).
\end{equation}

Note:

\[
    \eta = 2\exp (-2m\epsilon^2),
\]

\noindent let $\epsilon$ to $\epsilon_\text{MC}$, from which we can solve for:



\begin{equation}
    \epsilon_{\text{MC}} = \sqrt{\frac{\log (2/\eta)}{2m}}.
    \label{Eq:epsilon_MC}
\end{equation}

Thus, with probability at least $1-\eta$, we have:

\begin{equation}
    |\widehat{Q}_m(h_t, r_t)-Q_t^{\pi_0}(h_t, r_t)| \leq \epsilon_{\text{MC}}.
\end{equation}

\subsection{Approximate Policy Improvement with Learned Gates}

In practical scenarios, the true action-value function $Q_t^{\pi_0} (h_t, r_t)$ cannot be obtained directly and can only be approximated via finite Monte Carlo rollouts; therefore, the ideal gating policy cannot be implemented.

Based on the Monte Carlo estimate $\widehat{Q}_m (h_t, r_t)$, we define the empirical gating operator:

\begin{equation}
    \widehat{g}_m(h_t, r_t) = \mathds{1}[\widehat{Q}_m(h_t, r_t)\geq T].
\end{equation}

We further introduce a learnable gating operator:

\[
    g_\phi(h_t, r_t) \approx \widehat{g}_m(h_t, r_t),
\]

\noindent where $\phi$ denotes the parameters of the gating model.

Accordingly, the resulting gated reasoning policy is:

\begin{equation}
    \pi_\phi(r_t \mid h_t) = \frac{\pi_0(r_t\mid h_t)g_\phi(h_t, r_t)}{Z_\phi(h_t)},
\end{equation}

\noindent where the normalization term is:

\[
    Z_\phi(h_t) = \mathbb{E}_{r_t\sim \pi_0}[g_\phi(h_t, r_t)].
\]

The discrepancy between the learned policy $\pi_\phi$ and the ideal gated policy $\pi_T^*$ originates from two sources:

\begin{itemize}
    \item Monte Carlo estimation error induced by finite rollout sampling;
    \item Approximation error of the learnable gating operator. 
\end{itemize}

From Section~\ref{Sec:Monte_Carlo_Policy_Evaluation}, with probability at least $1-\eta$, the Monte Carlo estimation error satisfies:

\begin{align*}
|\widehat{Q}_{m}(h_t, r_t) - Q_t^{\pi_0}(h_t, r_t)| &\leq \epsilon_{\text{MC}}, \\
\text{where} \quad \epsilon_{\text{MC}} &= \sqrt{\frac{\log (2/\eta)}{2m}}.
\end{align*}

Meanwhile, we define the gating approximation error as:

\begin{equation}
    \epsilon_{\text{gate}} = \Pr[g_\phi (h_t, r_t) \neq \widehat{g}_m(h_t, r_t)].
\end{equation}

Under bounded terminal returns and finite reasoning horizon $K$, the policy improvement of the learned gated policy can be lower bounded by:

\begin{equation}
    J(\pi_\phi)-J(\pi_0) \geq \Delta^*-\delta_1 K \epsilon_{\text{MC}}-\delta_2 K \epsilon_{\text{gate}},
    \label{Eq:pi_phi_p_0}
\end{equation}

\noindent where,

\[
    \Delta^*=J(\pi_T^*)-J(\pi_0),
\]

\noindent denotes the ideal policy improvement obtained by oracle gating.

Substituting Eq.~\ref{Eq:epsilon_MC} into Eq.~\ref{Eq:pi_phi_p_0}, we obtain:

\begin{equation}
    J(\pi_\phi)-J(\pi_0) \geq \Delta^*-\delta_1 K \sqrt{\frac{\log (2/\eta)}{2m}}-\delta_2 K \epsilon_{\text{gate}}.
    \label{Eq:oracle_error}
\end{equation}

Therefore, as the number of Monte Carlo rollouts increases and the gating approximation error decreases, the learned gated reasoning policy progressively approaches the ideal value-gated policy improvement.

\section{Open-ended Reasoning Space Reshaping}
\label{appendix:open-ended_reasoning_space_reshaping}
The preceding analysis has demonstrated, from the perspective of improving finite-horizon policies, that the expected trajectory return of the ideal value-gated policy is no less than that of the original reasoning policy. In this section, we further characterize the role of step-wise gating from the perspective of an open-ended reasoning space.

Specifically, by restricting the admissible actions at each reasoning state, step-wise gating contracts the support of the complete trajectory distribution and redistributes probability mass among the retained trajectories. Consequently, the original open-ended reasoning space is reshaped into a value-admissible reasoning subspace characterized by the action-value function.

\subsection{Open-ended reasoning trajectory space}
Let 

\[
    \mathcal{T}=\{\tau=(r_1, \cdots, r_{K_\tau}): 1\leq K_\tau \leq K\},
\]

\noindent denote the open-ended space of all terminating reasoning trajectories, where $K_\tau$ is the termination depth of trajectory $\tau$, and $K$ is the maximum reasoning horizon.

The trajectory distribution induced by an arbitrary reasoning policy $\pi$ on $\mathcal{T}$ is

\begin{equation}
    p_\pi(\tau) = \prod_{t=1}^{K_\tau}\pi(r_t\mid h_t),
\end{equation}

\noindent where $h_t=(x, q, r_{1:t-1})$ denotes the reasoning state before the generation of the $t$th reasoning step.

Specifically, the trajectory distribution induced by the initial reasoning policy $\pi_0$ is

\begin{equation}
    p_0(\tau) = \prod_{t=1}^{K_\tau}\pi_0(r_t\mid h_t).
\end{equation}

Since reasoning actions are generated autoregressively in natural language, the model can produce highly diverse subsequent content at each step, resulting in a trajectory space $\mathcal{T}$ that is both open-ended and vast. The original policy assigns non-zero probability mass to trajectories with different terminal returns, so its support typically encompasses both high-return and low-return reasoning trajectories.

Recall the action-value function defined earlier:

\begin{equation}
    Q_t^{\pi_0}(h_t, r_t)=\mathbb{E}_{\tau\sim\pi_0}[R(\tau)\mid h_t, r_t],
\end{equation}

Given a value threshold $T$, define the ideal gate function as follows:

\begin{equation}
    g_T^*(h_t, r_t) = \mathds{1}[Q_t^{\pi_0}(h_t, r_t)\geq T].
\end{equation}

The corresponding ideal value-gated policy is

\begin{equation}
    \pi_T^*(r_t\mid h_t) = \frac{\pi_0(r_t\mid h_t)g_T^*(h_t, r_t)}{Z_T(h_t)},
\end{equation}

\noindent where

\begin{equation}
    Z_T(h_t)=\mathbb{E}_{r_t\sim\pi_0(\cdot\mid h_t)}[g_T^*(h_t, r_t)].
\end{equation}

Equivalently

\begin{equation}
    Z_T(h_t)=\Pr_{r_t\sim \pi_0(\cdot\mid h_t)}(Q_t^{\pi_0}(h_t, r_t)\geq T).
\end{equation}

Therefore, $Z_T(h_t)$ represents the total probability mass assigned by the original policy to all admissible actions in state $h_t$. After dividing by $Z_T(h_t)$, the probabilities of the retained actions are renormalized to form a valid distribution.



We assume that the admissible-action probability mass is uniformly positive at every state reachable under the original policy, \textit{i.e.},
\[
Z_T(h_t)>0
\qquad
\text{for all }h_t\text{ such that }p_0(h_t)>0.
\]

We further define the Oracle value-admissible trajectory space as

\begin{equation}
    \mathcal{T}_T^* = \{\tau \in \mathcal{T}: Q_t^{\pi_0}(h_t, r_t) \geq T, \quad \forall t \leq K_\tau\}.
\end{equation}

This set consists of all complete trajectories that satisfy the value constraint at every reasoning step.

\begin{proposition}
\label{appendix: proposition3}
\textbf{Support contraction of the reasoning trajectory space.}

Ideal step-wise gating restricts the support of the trajectory distribution to the Oracle value-admissible trajectory space, as follows:

\begin{equation}
    \mathrm{Supp}(p_T^*)\subseteq \mathcal{T}^*_T \subseteq \mathcal{T},
\end{equation}

\noindent where

\begin{equation}
    p_T^*(\tau) = \prod_{t=1}^{K_\tau} \pi_T^*(r_t \mid h_t).
\end{equation}
\end{proposition}

\begin{proof}
    For any state-action pair $(h_t, r_t)$, if

\[
    Q_t^{\pi_0}(h_t, r_t) < T,
\]

\noindent according to the definition of Oracle Gate,

\[
    g_T^*(h_t, r_t) = 0,
\]

\noindent consequently,

\[
    \pi_T^*(r_t \mid h_t) = 0.
\]

Consider now any complete reasoning trajectory $\tau = (r_1, \cdots, r_{K_\tau})$ that satisfies $p_T^*(\tau) > 0$.


Since the probability of a trajectory factorizes into the product of step-wise conditional probabilities, $p_T^*(\tau)>0$ implies
\[
\pi_T^*(r_t\mid h_t)>0,
\qquad \forall t\le K_\tau.
\]

This means that for any $t \leq K_\tau$, we have

\[
    g_T^*(h_t, r_t) = 1,
\]

\noindent that is,

\[
    Q_t^{\pi_0}(h_t, r_t) \geq T.
\]

Based on the definition of $\mathcal{T}_T^*$, we have

\[
    \tau \in \mathcal{T}_T^*.
\]

Therefore

\[
    \mathrm{Supp}(p_T^*) \subseteq \mathcal{T}_T^*.
\]

On the other hand, $\mathcal{T}_T^*$ is a subset obtained by applying stepwise value constraints to the original trajectory space $\mathcal{T}$; therefore

\[
    \mathcal{T}_T^* \subseteq \mathcal{T}.
\]

In summary,

\begin{equation}
    \mathrm{Supp}(p_T^*) \subseteq \mathcal{T}_T^* \subseteq \mathcal{T}.
\end{equation}

\end{proof}

\subsection{Trajectory-distribution reshaping}

For any $\tau \in \mathcal{T}_T^*$, every action in the trajectory passes through the Oracle Gate; therefore

\[
    g_T^*(h_t, r_t) = 1, \quad \forall t \leq K_\tau.
\]

Consequently

\begin{equation}
    p_T^*(\tau) = \prod _ {t=1}^{K_\tau} \pi_T^*(r_t\mid h_t) = \prod _{t=1}^{K_\tau} \frac{\pi_0(r_t \mid h_t)}{Z_T(h_t)} = p_0(\tau) \prod _{t=1}^{K_\tau} \frac{1}{Z_T(h_t)}.
\end{equation}

For any $\tau \notin \mathcal{T}_T^*$, there is at least one step in the trajectory that does not satisfy the value constraint; therefore, $p_T^*(\tau) = 0$.

Therefore, the distribution of trajectories after gating can be uniformly expressed as

\begin{equation}
    p_T^*(\tau) = p_0(\tau) \prod_{t=1}^{K_\tau} \frac{\mathds{1}[Q_t^{\pi_0}(h_t, r_t)\geq T]}{Z_T(h_t)}.
\end{equation}

This formulation indicates that step-wise gating has two effects on the distribution of complete trajectories.

First, the generation probability of any trajectory containing a reasoning step with an action-value below the threshold is set to zero.

Second, the probabilities of the retained admissible trajectories are rescaled according to the normalization factors $Z_T(h_t)$ associated with the states encountered along each trajectory. Since different trajectories may traverse different reasoning states, their corresponding normalization factors may also differ. Consequently, the gated trajectory distribution is not obtained by globally conditioning the original distribution on $\mathcal{T}_T^*$; rather, it results from a state-dependent, step-wise reweighting process.

In other words, step-wise gating not only contracts the support of the trajectory distribution but also alters the relative probability mass assigned to the retained trajectories. This captures the distributional interpretation of reasoning-space reshaping.

\subsection{Expected-return preservation under space reshaping}

Combining Proposition~\ref{appendix: proposition3} with the finite-horizon policy improvement result established in Section~\ref{app:proof_gating_monotonicity}, we obtain

\begin{equation}
    J(\pi_T^*) \geq J(\pi_0),
\end{equation}

\noindent where

\[
    J(\pi) = \mathbb{E}_{\tau\sim p_\pi}[R(\tau)].
\]

Equivalently,

\begin{equation}
    \sum_{\tau\in\mathcal{T}}p_T^*(\tau)R(\tau) \geq \sum_{\tau\in\mathcal{T}}p_0(\tau)R(\tau).
\end{equation}

Therefore, the value-admissible trajectory space induced by Oracle gating is not an arbitrary subspace of the original trajectory space. Rather, it is explicitly characterized by the action-value function under the original policy, and the resulting trajectory distribution preserves or improves the expected return of complete reasoning trajectories. 
Therefore, the value-admissible trajectory space induced by Oracle gating is not an arbitrary subspace of the original trajectory space. Rather, it is characterized by the action-value function under the original policy, and its induced trajectory distribution preserves or improves the expected return of complete reasoning trajectories.

\subsection{Approximate reshaping under a learnable gate}

In practice, the true action-value function $Q_t^{\pi_0}(h_t, r_t)$ cannot be obtained directly. This paper uses reasoning-tree rollouts to construct a Monte Carlo estimate:

\[
    \hat{Q}_m(h_t, r_t) = \frac{1}{m} \sum_{j=1}^m R(\tau^{(j)}), \quad \tau^{(j)}\sim^\mathrm{i.i.d.} p_{\pi_0}(\cdot \mid h_t, r_t).
\]

Based on these Monte Carlo estimates, we train a learnable gating function $g_\phi(h_t, r_t)$ to approximate the Oracle gate induced by the true action-value function.

The corresponding learned gated policy is

\begin{equation}
    \pi_\phi(r_t \mid h_t) = \frac{\pi_0(r_t\mid h_t)g_\phi(h_t, r_t)}{Z_\phi(h_t)},
\end{equation}

\noindent where

\[
    Z_\phi(h_t) = \mathbb{E}_{r_t\sim \pi_0(\cdot \mid h_t)}[g_\phi(h_t, r_t)].
\]

Define the expected return gain for Oracle gating as

\[
    \Delta^* = J(\pi_T^*) - J(\pi_0).
\]

According to the preceding approximate policy improvement analysis, under the assumptions that the terminal return is bounded, the reasoning horizon is finite, and the gating normalization term is uniformly bounded away from zero, we have

\begin{equation}
    J(\pi_\phi) - J(\pi_0) \geq \Delta^* - C_1 K \epsilon_{\mathrm{MC}} - C_2 K \epsilon_\mathrm{gate},
\end{equation}

\noindent where $\epsilon_\mathrm{MC}$ represents the value estimation error generated by finite Monte Carlo rollouts, $\epsilon_\mathrm{gate}$ represents the approximation error of the learned gate relative to the Oracle Gate, and $C_1$ and $C_2$ are constants related to the normalization lower bound and the value-margin condition.

Therefore, when

\[
    \Delta^* > C_1 K \epsilon_\mathrm{MC} + C_2 K \epsilon_\mathrm{gate},
\]

\noindent we have

\[
    J(\pi_\phi) > J(\pi_0).
\]

Furthermore, when

\[
    \epsilon_\mathrm{MC} \rightarrow 0, \quad \epsilon_\mathrm{gate} \rightarrow 0,
\]

\noindent if the learned policy approximates the Oracle policy on the state-conditional distribution, then in a finite horizon, its trajectory distribution satisfies

\[
    \mathrm{TV} (p_\phi, p_T^*) \rightarrow 0.
\]

Due to

\[
    R(\tau) \in [0, 1],
\]

\noindent we have

\[
    |J(\pi_\phi) - J(\pi_T^*)| \leq \mathrm{TV} (p_\phi, p_T^*),
\]

\noindent consequently

\[
    J(\pi_\phi) \rightarrow J(\pi_T^*).
\]

As the Monte Carlo value estimates become more accurate and the learned gate more closely approximates the Oracle gate, the learned gating policy increasingly inherits the reasoning-space contraction, trajectory-distribution reshaping, and expected-return improvement properties of Oracle gating.

\textbf{Summary of Theoretical Findings.}

In summary, step-wise value gating reshapes the open-ended reasoning space at two levels: the support and the probability distribution.

At the support level,

\[
    \mathrm{Supp}(p_T^*) \subseteq \mathcal{T}_T^* \subseteq \mathcal{T},
\]

\noindent the gated trajectory distribution covers only those trajectories that satisfy the step-wise value constraints.

At the distribution level, 

\[
    p_T^*(\tau) = p_0(\tau) \prod_{t=1}^{K_\tau} \frac{\mathds{1}[Q_t^{\pi_0}(h_t, r_t) \geq T]}{Z_T(h_t)},
\]

\noindent illustrates how step-wise gating applies state-dependent, incremental reweighting to the trajectory distribution.

Based on the results of the previously mentioned policy improvements, the trajectory distribution after gating satisfies

\[
    J(\pi_T^*) \geq J(\pi_0).
\]

Therefore, step-wise gating does not merely filter individual reasoning actions in isolation. Instead, it reshapes the original open-ended trajectory distribution into a value-admissible distribution characterized by the action-value function and having an expected return no lower than that of the original policy.

\section{Computational Complexity and Inference Cost}

We further analyze the computational overhead of different inference enhancement methods. Let $N$ be the number of problems, $M$ the number of candidates generated per round, and $K$ the maximum number of inference steps. $C^{\mathrm{traj}}_{\mathrm{gen}}$ and $C^{\mathrm{step}}_{\mathrm{gen}}$ represent the costs of generating a complete inference trajectory and a single-step continuation, respectively. Furthermore, $C_{\mathrm{vote}}$ represents the cost of using an expert LLM/VLM to evaluate or vote on a candidate, and $C_{\mathrm{gate}}$ represents the cost of the gating model assigning a value score to a candidate. Typically, there are

\[
    C^{\mathrm{traj}}_{\mathrm{gen}} \approx C^{\mathrm{step}}_{\mathrm{gen}}, \quad C_{\mathrm{gate}} \ll C_{\mathrm{vote}},
\]

\noindent where the gated model performs only forward computations of the discriminator with a smaller parameter size, and candidate scoring can be performed in batches; in contrast, expert model evaluation
Typically, it involves larger models and additional autoregressive decoding.

\textbf{Direct CoT} generates a complete reasoning trail for each problem, with a time complexity of

\[
    \mathcal{O}(NKC^{\mathrm{traj}}_{\mathrm{gen}}).
\]

\textbf{Stepwise} reasoning breaks down a complete trajectory into at most $K$ stepwise generations, so its complexity is

\[
    \mathcal{O}(NKC^{\mathrm{step}}_{\mathrm{gen}}).
\]

The two approaches are comparable in terms of the number of tokens generated. Still, Stepwise requires multiple rounds of model calls and repeated context encoding, so it typically results in higher actual latency.

\textbf{Self-Consistency} independently samples $M$ complete inference trajectories and aggregates the final answers; its main computational overhead is

\[
    \mathcal{O}(NKMC^{\mathrm{traj}}_{\mathrm{gen}}).
\]

The additional computations involved in this method are almost entirely devoted to expanding the sampling space, while there is no explicit control over the quality of intermediate candidates.

\textbf{In ToT}, each inference step generates $M$ candidate continuations, and an evaluation model selects the subsequent extension branch. Therefore, its complexity can be expressed as

\[
    \mathcal{O}(NKMC^{\mathrm{step}}_{\mathrm{gen}}+NKMC_{\mathrm{vote}}).
\]

If a single model invocation evaluates all candidates for the current step simultaneously, the number of evaluation invocations can be written as $NK$; however, its input length and computational complexity still scale with $M$. Therefore, under token-level computational overhead, it remains approximately linearly dependent on the number of candidates. The additional voting cost for the final answer pool is $\mathcal{O}(NMC_{\mathrm{vote}})$, which can be absorbed by the main term mentioned above. The main additional overhead of ToT stems from node evaluation. Although the base model can be used directly for self-evaluation, weaker models often struggle to reliably assess the value of complex spatial reasoning steps. To ensure the quality of branch selection, practical applications typically require invoking a more capable LLM or VLM as an evaluator, making $C_{\mathrm{vote}}$ a non-negligible cost.

\begin{figure*}[t]
    \centering
    \includegraphics[width=0.8\linewidth]{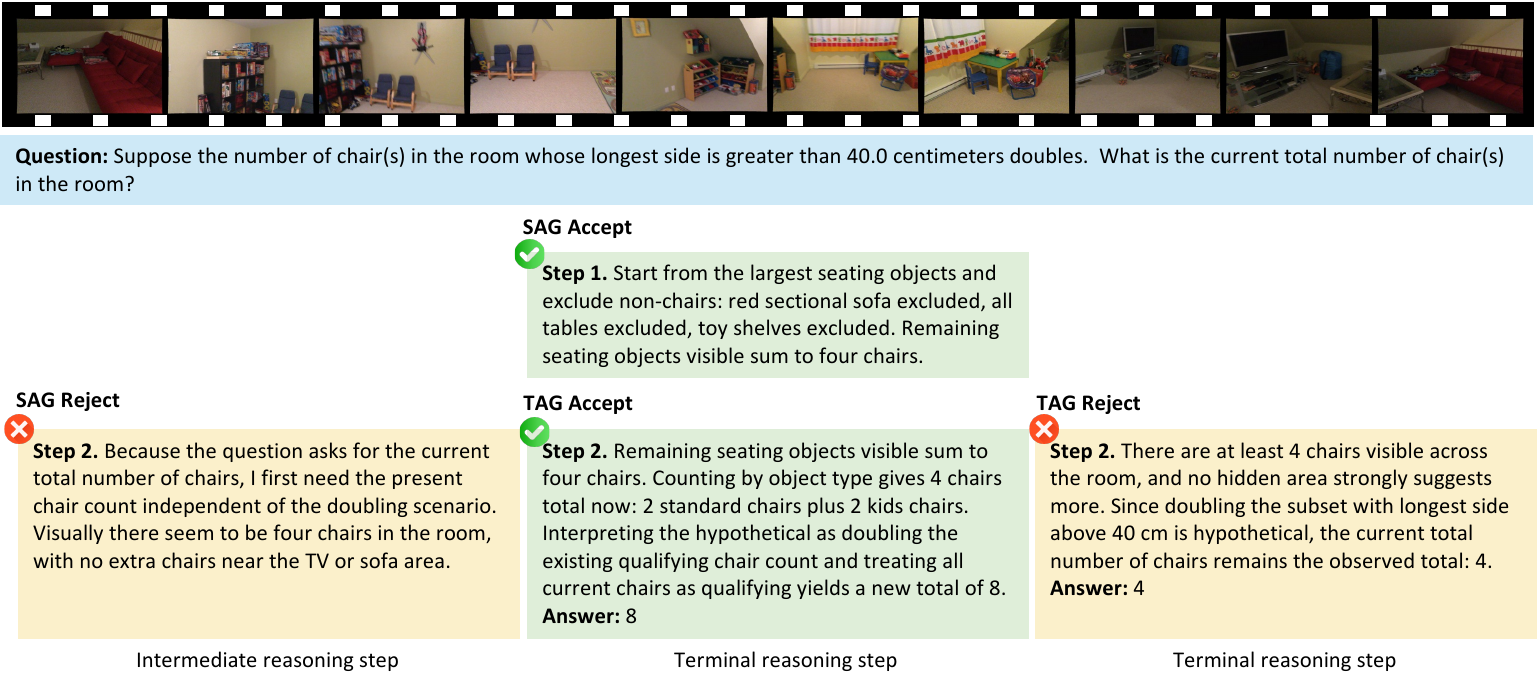}
    \caption{Qualitative illustration of hierarchical advantage-guided gating on a object-counting example. SAG accepts the visually grounded prefix that correctly identifies four chairs, while rejecting an intermediate step that incorrectly decouples the observed count from the hypothetical doubling operation. TAG subsequently evaluates each complete reasoning chain together with its final answer, accepting the trajectory that consistently applies the counterfactual transformation and outputs 8, while rejecting the locally plausible but globally inconsistent trajectory ending in 4. Green and red denote accepted and rejected reasoning candidates, respectively.}
    \label{fig: qualitative}
\end{figure*}

\textbf{The AG} and ToT methods described in this paper have similar candidate-generation complexities: each produces $M$ candidates per step; the difference is that AG uses a trained SAG to score intermediate prefixes, while TAG selects the final answer. Therefore, its complexity is

\[
    \mathcal{O}(NKMC^{\mathrm{step}}_{\mathrm{gen}}+NKMC_{\mathrm{gate}}),
\]

\noindent where the main term can also absorb the TAG scoring cost of the final answer pool. From the perspective of the asymptotic order of $N$, $K$, and $M$, both AG and ToT belong to the $\mathcal{O}(NKM)$; however, there is a significant difference in their actual costs during the selection phase. AG’s gated model can be deployed locally and evaluate multiple candidates in a single discriminative forward pass, eliminating the need for repeated calls to large-scale expert models or for generating evaluation text. Since $C_{\mathrm{gate}} \ll C_{\mathrm{vote}}$, AG significantly reduces the computational and invocation costs of value evaluation and branch selection while retaining the ability to perform incremental candidate search.

It should be noted that AG’s efficiency advantage is primarily evident in the candidate selection phase; its candidate generation overhead is on the same order of magnitude as ToT’s and does not rely on reducing the search space to achieve acceleration. Under conditions where the generation length and candidate budget are fixed, the actual computational overhead typically satisfies:

\[
    \text{Direct CoT} < \text{Stepwise} < \text{Self-Consistency} \lesssim \text{AG} < \text{ToT}.
\]

Among these, Self-Consistency and AG have roughly equivalent computational overhead; AG incurs only a relatively lightweight gating-score cost, while ToT incurs additional overhead from evaluating candidates using a strong expert model. Consequently, while AG has a computational cost slightly higher than that of simple repeated sampling, it achieves explicit quality control over the candidate solution space. It selects more reliable inference paths at a selection cost significantly lower than that of an expert model search.

\section{Qualitative Analysis of Advantage-Guided Gating}
Figure~\ref{fig: qualitative} illustrates how SAG and TAG intervene at different stages of the reasoning process. The SAG-accepted prefix first excludes visually similar but semantically irrelevant objects, such as sofas, tables, and shelves, and correctly identifies four chairs in the scene. In contrast, the rejected intermediate step treats the current chair count as independent of the hypothetical doubling operation. Although its visual observation of four chairs is plausible, this interpretation is logically inconsistent with the counterfactual condition in the question. Once retained, such a semantic deviation would condition subsequent generation on an incorrect problem formulation and propagate the error along the remaining trajectory. SAG therefore serves as a prospective prefix filter: it preserves reasoning states that remain logically consistent with the task and suppresses locally plausible but directionally incorrect steps before they are further expanded.

TAG operates at the terminal level and jointly evaluates the complete reasoning trajectory and its predicted answer. In this example, both terminal candidates are based on the same visual count of four chairs, but they differ in how the counterfactual condition is applied. The accepted trajectory consistently connects the visual evidence, the doubling operation, and the final answer of 8. By contrast, the rejected trajectory concludes that the total remains 4, producing an answer that is visually plausible but inconsistent with its required transformation. This demonstrates that TAG does not merely assess answer plausibility in isolation; instead, it considers both trajectory-level logical consistency and final-answer correctness. The two gates thus play complementary roles: SAG prevents early logical errors from contaminating subsequent reasoning, whereas TAG identifies residual errors that become evident only after examining the complete chain and answer. Their combination provides hierarchical control from intermediate reasoning preservation to final trajectory selection.

\end{document}